\documentclass[12pt]{article}
\usepackage[margin=1in]{geometry}

\usepackage[affil-it]{authblk} 
\usepackage{etoolbox}
\usepackage{lmodern}
\usepackage{rotating}
\makeatletter
\patchcmd{\@maketitle}{\LARGE \@title}{\fontsize{14}{17.2}\selectfont\@title}{}{}
\makeatother

\usepackage{sectsty}

\sectionfont{\fontsize{13}{15}\selectfont}
\subsectionfont{\fontsize{12}{15}\selectfont}

\usepackage[font=footnotesize]{caption}
\usepackage{authblk}
\usepackage{microtype}
\usepackage{graphicx}
\usepackage{subfigure}
\usepackage{booktabs} 
\usepackage{amsmath,amssymb, amsthm, bm}
\usepackage{color,soul}
\DeclareMathOperator*{\argmax}{argmax}
\usepackage{setspace}
\usepackage{gensymb}
\usepackage{bbm}
\usepackage[colorlinks=true,allcolors=blue]{hyperref}
\usepackage{natbib}
\usepackage{url}
\usepackage{wrapfig}
\usepackage{comment}
\usepackage{amsmath,amsfonts,amssymb}

\usepackage{tensor}

\usepackage{mathtools}

\usepackage[ruled,vlined]{algorithm2e}
\usepackage{amsmath}
\usepackage[makeroom]{cancel}

\DeclarePairedDelimiterX{\infdivx}[2]{(}{)}{%
  #1\;\delimsize\|\;#2%
}

\newcommand{\norm}[1]{\left\lVert#1\right\rVert}

\newtheorem{theorem}{Theorem}
\newtheorem{assumption}{Assumption}

\newtheorem{subassumption}{Assumption\normalfont}[assumption]
\newenvironment{assumption*}
 {\ifnum\value{subassumption}=0 \stepcounter{assumption}\fi\subassumption}
 {\endsubassumption}
\newenvironment{assumption+}[1]
 {\renewcommand{\thesubassumption}{#1}\subassumption}
 {\endsubassumption}

\newtheorem{lemma}{Lemma}

\newtheorem{corollary}{Corollary}

\usepackage{colortbl}
\newtheorem*{theorem-non}{Theorem}

\usepackage{amsmath}

\usepackage{fontawesome5}
\usepackage{hyperref}
\makeatletter
\newcommand{\github}[1]{%
   \href{#1}{\faGithubSquare}%
}
\makeatother

\title{\textbf{Collaborative and Distributed Bayesian Optimization via Consensus: \\ Showcasing the Power of Collaboration for Optimal Design}}
\date{\vspace{-7ex}}

\author[1]{Xubo Yue}
\author[3]{Yang Liu}
\author[2]{Albert S. Berahas}
\author[3]{Blake N. Johnson}
\author[2,*]{Raed Al Kontar}
\affil[1]{Northeastern University, Boston}
\affil[2]{University of Michigan, Ann Arbor}
\affil[3]{Virginia Tech, Blacksburg}
\affil[*]{Corresponding author: alkontar@umich.edu}

\begin{document}
\maketitle

\begin{abstract}

Optimal design is a critical yet challenging task within many applications. This challenge arises from the need for extensive trial and error, often done through simulations or running field experiments. Fortunately, sequential optimal design, also referred to as Bayesian optimization when using surrogates with a Bayesian flavor, has played a key role in accelerating the design process through efficient sequential sampling strategies. However, a key opportunity exists nowadays. The increased connectivity of edge devices sets forth a new collaborative paradigm for Bayesian optimization. A paradigm whereby different clients collaboratively borrow strength from each other by effectively distributing their experimentation efforts to improve and fast-track their optimal design process. To this end, we bring the notion of consensus to Bayesian optimization, where clients agree (i.e., reach a consensus) on their next-to-sample designs. Our approach provides a generic and flexible framework that can incorporate different collaboration mechanisms. In lieu of this, we propose transitional collaborative mechanisms where clients initially rely more on each other to maneuver through the early stages with scant data, then, at the late stages, focus on their own objectives to get client-specific solutions. Theoretically, we show the sub-linear growth in regret for our proposed framework. Empirically, through simulated datasets and a real-world collaborative sensor design experiment, we show that our framework can effectively accelerate and improve the optimal design process and benefit all participants.

\end{abstract}

\newpage

\section{Introduction}

The success of many real-world applications critically depends on trial \& error. Often the goal is to manipulate a set of variables, called \textit{designs}, to achieve a desired outcome or system response.  For instance, material scientists perform time-consuming and expensive experiments to determine optimal compositions \citep{zhang2023rapid} that produce a material with desired properties. 
% For instance, {\color{red}material scientists perform time-consuming and expensive experiments to determine optimal sensor design \citep{zhang2024improving} that finds the maximum amount of captured target analyte}. 
Similarly, additive manufacturers must calibrate many design parameters, including laser power, beam diameter, and hatching pattern through trial \& error \citep{buchanan2019metal} so that their product matches its intended shape. Those studies, whether through experimentation or simulations like finite element models, consume resources (e.g., time and budget) that can significantly limit progress.

%\citep{shahriari2015taking}

Fortunately, sequential design has played a key role in accelerating the design process \citep{le2015cokriging,  kusne2020fly, huang2021bayesian}. In the engineering and statistical literature, sequential design can be categorized into two types \citep{ezzat2018sequential}: (i) \textit{sequential design for exploration}, that aims to explore a response surface, (ii) \textit{sequential design for optimization}, that aims to find designs that optimize a target response. Our work focuses on sequential design for optimization, known as sequential optimal design. This field has also been coined more recently as Bayesian optimization (BO) when using surrogates with a Bayesian flavor (e.g., Gaussian processes - $\mathcal{GP}$s). Hereon, we will use the term Bayesian optimization since, without loss of generality, we exploit $\mathcal{GP}$s in the process. 

Rather than running exhaustive brute-force experiments over a dense grid, BO employs a sequential strategy to conduct experiments and observe new samples. It first fits a surrogate that estimates the design-response relationship from existing data. Afterward, a utility, also known as acquisition, function \citep{gramacy2009adaptive,overstall2017bayesian, lee2018sequential} is defined to hint at the benefits of sampling/experimenting at new design points. Based on this utility metric, the next-to-sample design point is selected. The procedure is sequentially iterated over several rounds till budget constraints or exit conditions are met. Needless to say, sequential optimal design/BO has been extensively studied \citep{antony2014design, shahriari2015taking, kleijnen2018design} and has found success in a wide variety of disciplines across physics, chemistry, mechanical engineering \citep{ekstrom2019bayesian, kusne2020fly,zhang2020bayesian,rahman2021predicting,wang2022bayesian}, amongst many others. 

In lieu of the aforementioned successes, this work aims to bring BO to a collaborative paradigm. The main question we ask is: \textbf{How can multiple clients collaborate to improve and fast-track their design processes}? With today's advances in computation and communication power at edge devices \citep{kontar2021internet}, it has become more plausible for potentially dispersed clients to share information, distribute trial \& error efforts, and fast-track the design process so that all participants gain benefit. Here, collaborating clients can be scientists, robots, multiple finite element simulations, etc.

However, to enable collaboration, key challenges exist. First, and most importantly, is how to distribute the sequential optimization process. While BO has been extensively studied in the past decades, few literature exists on collaborative BO. The second challenge is heterogeneity. Despite trying to optimize similar processes, clients may operate under different external factors and conditions. As such, retaining client-specific optimal solutions is of importance. The third challenge is privacy. To encourage participation, a collaborative process should refrain from sharing client-specific outcomes.

In an effort to address the opportunity and stated challenges above, we summarize our contributions below: 

\begin{itemize}
    \item \textbf{Bayesian optimization via Consensus:} We bring the notion of ``consensus'' to BO where clients perform experiments locally and agree (i.e., reach a consensus) on their individualized next-to-sample designs. Our approach provides a generic and flexible framework that can incorporate different collaboration mechanisms. It hinges upon a consensus matrix that evolves with iterations and flexibly determines the next-to-sample designs across all clients. % Through this matrix, in the early stages, clients rely more on each other to borrow strength and information and maneuver through stages with scant data. In the late stages, each client will gradually shift focus to their individual objective. 
    Some interesting by-products of our framework are:
    \begin{enumerate}
        \item Amenability to full decentralization that respects communication restrictions and allows clients to perform the collaborative process without the orchestration of a central entity. 
        \item Ability to only share designs and not client-specific outcomes to preserve some privacy.
        \item Amenability to any utility/acquisition function of choice.
    \end{enumerate}
  
    \item \textbf{Regret minimization:} We support our proposed collaborative framework with a theoretical justification. We show that, under some mild conditions, with a high probability, the cumulative regret for each client has a sublinear growth rate. This implies that our collaborative algorithm can bring clients within regions of optimal designs.     
    \item \textbf{The power of collaboration for optimal sensor design:} We perform an experiment for accelerated optimization of sensor design through collaborative finite element analysis (FEA) workflows. The experiment showcases the ability of collaboration to significantly outperform its non-collaborative counterpart.
    
    % Furthermore, through extensive simulations, we highlight the advantageous properties of our approach.   
\end{itemize}

The remainder of this paper is organized as follows. In Sec. \ref{sec:motivation}, we present a motivating example that drives this work. We then briefly discuss related literature in Sec. \ref{sec:lit}. In Sec. \ref{sec:colab}, we formulate the problem mathematically and provide our proposed collaborative BO framework. Sec. \ref{sec:exp} and Sec. \ref{sec:real} showcase the benefits of our approach through several simulation studies and an experiment on accelerated sensor design.  We conclude our paper with a brief discussion about open problems in this relatively new field in Sec. \ref{sec:con}.

\section{Motivation}
\label{sec:motivation}

Our work is motivated by optimal sensor design. Biosensors are critical bioanalytical technologies that enable the selective detection of target analytes and have broad applications ranging from medical diagnostics, health monitoring, food and water safety, and environmental monitoring \citep{cesewski2020electrochemical}.  A biosensor is defined as a device that is based on an integrated biorecognition element and transducer \citep{thevenot2001electrochemical}.  There are two main categories of biosensors: (1) device-based biosensors, and (2) methods-based biosensors (e.g., nanobiosensors) \citep{johnson2014biosensor}.  Device-based biosensors are often integrated with microfluidics and exhibit form factors that can be physically integrated with and removed from processes (e.g., thin-film, dip-stick) and preserve the characteristics of the sample to be analyzed. Alternatively, methods-based biosensors based on solutions and suspensions of functionalized particles provide detection by mixing with a sample, and thus, are relatively destructive with respect to the sample. Thus, device-based biosensors are commonly used in a ``flow-and-measure” or ``dip-and-measure” format. In contrast, methods-based biosensors, such as nanobiosensors, are typically used in a “mix-and-measure" format.

While biosensors have been created for sensitive and selective detection of many target analytes ranging from small molecules, proteins, nucleic acids, cells, biomarkers, and pathogens, there remains a need to further optimize biosensor design (e.g., form factor, functionalization) and performance (e.g., sensitivity, measurement confidence, and speed) to meet the constraints and requirements of industrial and commercial applications \citep{carpenter2018blueprints}.  Current thrust areas in biosensor design and performance optimization can be categorized as driven by experimentation, simulation, data analytics, or combinations thereof \citep{selmi2017optimization, polatouglu2020novel,koveal2022high, zhang2023reduction, zhang2024improving}.  For example, high-throughput experimentation, sensor arrays, FEA, and data-driven biosensing have been leveraged to improve the understanding, design, and performance of biosensors, particularly device-based biosensors whose design, utility, and performance are often linked to the characteristics of experimental measurement formats, such as microfluidic channel design, flow field parameters (e.g., flow rate), and other parameters of the measurement format (e.g., sample injection time). However, it remains a challenge to optimize biosensor performance, particularly in maximizing the amount of target analyte detected, given the high-dimensionality of the design space associated with biosensor design, functionalization (e.g., concentration of immobilized biorecognition elements), and measurement format parameters. In particular, there remains a need to establish closed-loop self-driving workflows for engineering high-performance biosensors, such as by optimizing the biosensor design, functionalization protocols, and measurement formats that synergize with experimentation, simulation, and machine learning. Given the current limitation, our central hypothesis is:

\textit{Hypothesis: Collaboration can accelerate the pace of optimal sensor design and yield optimal design with minimal resource expenditures.}

Fig. \ref{fig:motivation} provides a microcosm of our collaborative solution, which we have tested in Sec. \ref{sec:real}. Our experiment features multiple collaborating agents (FEA simulations) that perform biosensor design. Then they will coordinate to decide on their next simulations. Our goal is to rapidly discover the biosensor design and measurement format parameters that find the the maximum amount of captured target analyte.

\begin{figure}[!htbp]
    \centering
    \centerline{\includegraphics[width=0.8\columnwidth]{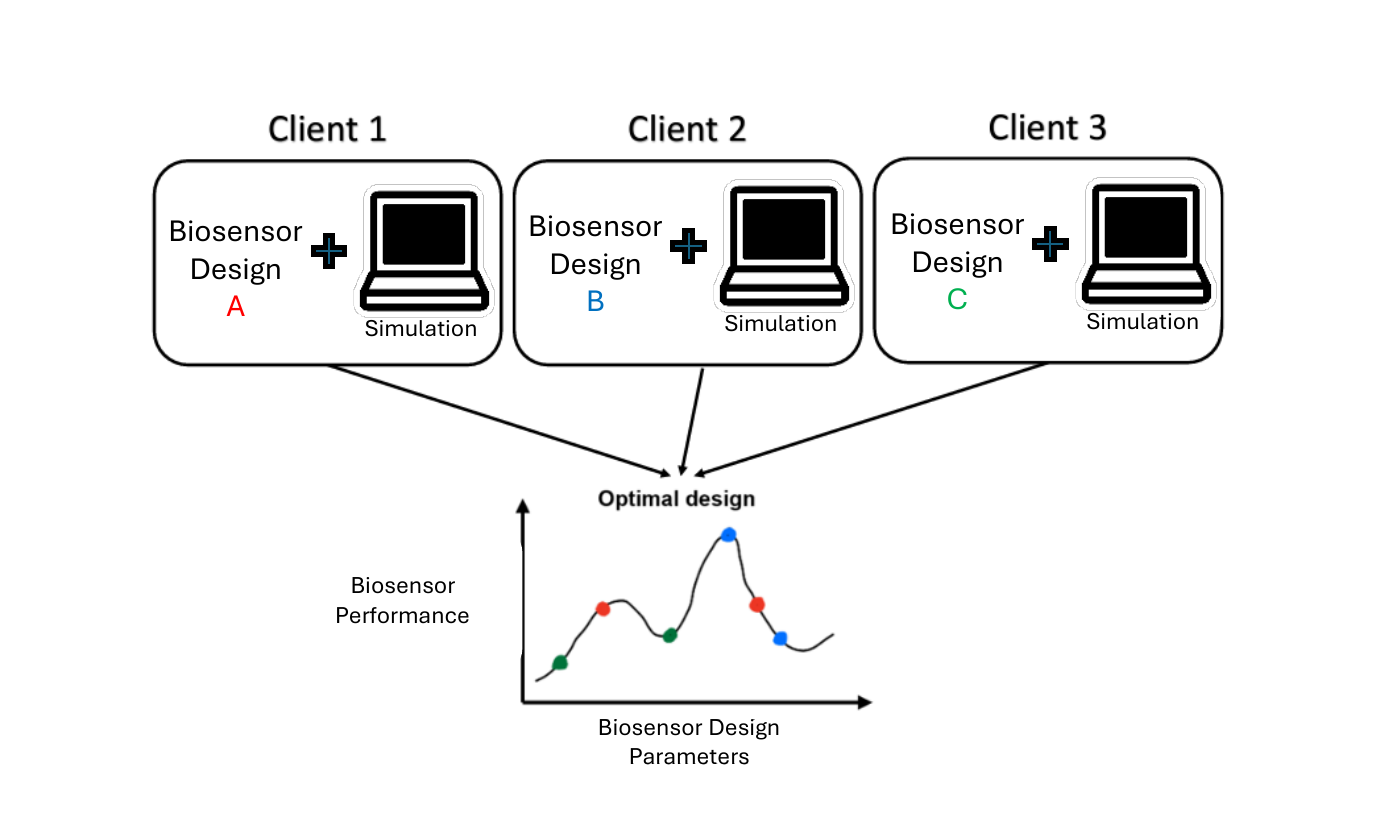}}
    \caption{Illustration of closed-loop biosensor and measurement format  design optimization driven by consensus BO-driven simulation. }
    \label{fig:motivation}
\end{figure}

\section{Related Work}
\label{sec:lit}

Optimal design has a longstanding history. Traditionally, the major trend for optimal design was fixed-sample designs, where the best design is chosen from a pre-determined set. Perhaps the earliest departures from fixed-sample designs date back to a few years following World War II, where \cite{friedman1947planning, box1951experimental,  robbins1951stochastic, bellman1956problem, chernoff1959sequential} proposed schemes to sequentially decide on next-to-sample designs to eventually reach the optimal design. Amongst them were the experiments done by E.P Box and co-workers \cite{box1951experimental} at the Imperial Chemical Industries, aimed at improving the yield of chemical processes. Afterward, two seminal papers from \cite{kushner1964new, sacks1970statistical} proposed deciding on the next-to-sample design points through optimizing utility functions driven by surrogates that can quantify uncertainty. This led to an era of development in BO that remains very active till this day.    

% hotelling1941experimental, 

% from the fixed-sample design dates back to 1929 when engineers were interested in inspecting the proportion of defectives in the manufacturing process \citep{dodge1929method}. Later on, the work of \cite{dodge1929method} spurred a lot of research interests along the line of the sequential design of experiments . 

Recent advances in BO include but are not limited to: (i) \textbf{Alternative surrogate models:} instead of using $\mathcal{GP}$ surrogates, \cite{marmin2022deep, ming2022deep, sauer2022active} suggest using Deep $\mathcal{GP}$s as an alternative. This credits to many desirable advantages of Deep $\mathcal{GP}$s, such as flexibility to non-stationarity and robustness in handling abrupt regime changes in the training data. Besides deep $\mathcal{GP}$s, Bayesian neural networks \citep{snoek2015scalable, springenberg2016bayesian} have also been widely adopted as surrogates.  (ii) \textbf{Multi-objective BO:} where the goal is to simultaneously optimize multiple, often competing, responses  \citep{konomi2014bayesian, svendsen2020active}. Here surrogates that simultaneously learn the multiple outputs, such as the Multi-output $\mathcal{GP}$ \citep{kontar2018nonparametric}, are often exploited and developed. (iii) \textbf{Multi-fidelity BO:} when data is collected across multiple fidelities, it becomes crucial to choose the fidelity to use when running an experiment. Along this line, \cite{le2015cokriging, he2017optimization, stroh2022sequential} have developed various sequential strategies to tackle this challenge. (iv) \textbf{New utility functions:} developing new utility functions remains one of the most active areas in optimal design. Along this line, recent literature has investigated look-ahead utilities that chose a design based on its utility over a rolling future horizon \citep{jiang2020binoculars, lee2020efficient, yue2020non}.

%Instead of finding one single optimal design point, level-set estimation algorithms return a set of design points whose response values are above (or below) a certain threshold such that all of these candidate design points yield desirable response outputs \citep{gotovos2013active, letham2022look}. 

Though BO has caught major attention over the past century, to our knowledge, little literature exists on collaborative BO. Perhaps the two closest fields are batch sequential design (or batch BO) and federated BO. In batch BO, multiple designs are chosen from a common surrogate and utility. These designs are then distributed across multiple compute nodes to be evaluated in parallel \citep{azimi2010batch,azimi2012hybrid, wu2016parallel,duan2017sliced, hunt2020batch}. Unfortunately, batch BO is designed for parallel computing, where it focuses on optimizing a single objective using a single surrogate learned from a centralized dataset. Therefore, it cannot handle cases where data comes from  diverse and potentially heterogeneous sources, nor does it preserve privacy. All data needs to be agglomerated in one place to learn the surrogate and optimize the utility. In a similar fashion, federated BO \citep{dai2020federated,dai2021differentially} tries to distribute experiments decided from a single objective. However, they do so while preserving privacy. For instance, \cite{dai2020federated} share function realizations of the posterior belief using random Fourier features to learn the common surrogate. While preserving privacy, federated BO tries to circumvent its inherent homogeneity assumptions by reducing collaboration after some time, yet such a method is limited to a Thomson-sampling utility. 

% In a similar fashion, federated BO \citep{dai2020federated,dai2021differentially} tries to distribute experiments decided from a single objective. However, they do so while preserving privacy. For instance, \cite{dai2020federated} share function realizations of the posterior belief using random Fourier features to learn the common surrogate. While preserving privacy, federated BO inherits the homogeneity assumptions of batch BO and is limited to a Thomson-sampling-based utility.

\section{The Collaborative BO Framework}
\label{sec:colab}

\subsection{Setting the Stage}
We start by describing our problem setting and introducing notation. Assume there are $K\geq 2$ clients, and each client has a budget of $T$ experiments across $T$ iterations. Denote by $t\in\{0,1\ldots,T-1\}$ the iteration index. Clients can communicate with each other either via a central orchestrator or direct communication (See Sec. \ref{subsec:collaborative}). 

Each client $k\in[K]\coloneqq\{1,\ldots K\}$ has an initial dataset $\mathcal{D}^{(0)}_k=\{\bm{X}^{(0)}_k,\bm{y}^{(0)}_k\}$ with $N^{(0)}_k$ observations, where $\bm{X}^{(0)}_k=\{\bm{x}_{k,1},\ldots,\bm{x}_{k,N^{(0)}_k}\}$ is a $D\times N^{(0)}_k$ design matrix that contains the initials designs $\bm{x}_{k, \cdot} \in \mathbb{R}^D$ and $\bm{y}^{(0)}_k=(y_{k,1},\ldots,y_{k,N^{(0)}_k})^\intercal$ is an $N^{(0)}_k\times 1$ vector that contains the corresponding observed responses. The goal of each client is to find a set of client-specific optimal designs 
\begin{align}
    \bm{x}^*_k=\argmax_{\bm{x}\in\mathcal{X}\subseteq\mathbb{R}^D} f_{k}(\bm{x}) \, , \notag
\end{align}

\noindent where $\mathcal{X}$ is a subset in $\mathbb{R}^D$, and $f_{k}: \mathbb{R}^D \rightarrow \mathbb{R}$ is the true unknown design function each client aims to optimize. Clearly, using first or second-order optimization algorithms is not feasible since $f_k$ is a black-box. To observe $f_k(\cdot)$ at a new design point $\bm{x}_k^{\text{new}}$ one needs to run an experiment and observe $y_{k}(\bm{x}_k^{\text{new}})$, that is possibly a noisy representation of $f_{k}(\bm{x}_k^{\text{new}})$; $y_k=f_k+\epsilon_k$, where $\epsilon_k$ is an additive noise. Therefore, the goal is to carefully decide on the next-to-sample design $\bm{x}_k^{\text{new}}$ so that an optimal design is reached with the fewest experiments possible.

To do so, at any time $t$, BO resorts to a utility function $U$, $U(y_k(\bm{x}); \mathcal{D}^{(t)}_k): \mathbb{R}^D\to\mathbb{R}$, that quantifies the benefits gained if one were to conduct an experiment at a new design, $\bm{x}_k^{(t)\text{new}}$ \citep{gramacy2009adaptive,konomi2014bayesian,amo2016optimal, overstall2017bayesian, lee2018sequential}. 

% For example, the expected improvement utility \citep{jones1998efficient} considers $\mathbb{E}_{y_k}\left[U_k(y_k(\bm{x}))\right]=\mathbb{E}_{y_k}\left[(y_k(\bm{x})-y_k^*)^+\right]$, where $a^+=\max(a,0)$, and $y_k^*=\max \bm{y}_k$ is the current best observation. 

In a non-collaborative environment, each client $k$ chooses the next-to-sample design by maximizing their own utility. However, since the utility is dependent on the response $y_k$, one cannot calculate the utility except at previously observed designs. Here BO resorts to a surrogate $\hat{y}_k$ that estimates $\hat{y}_k(\bm{x})$ for any $\bm{x} \in \mathbb{R}^D $ using the dataset $\mathcal{D}_k^{(t)}$. Such surrogates are often probabilistic (e.g., $\mathcal{GP}$ \citep{sacks1989design, currin1991bayesian}) and are capable of providing a predictive distribution $\mathbb{P}_{\hat{y}_k(\bm{x})|\mathcal{D}^{(t)}_k}$ over $\hat{y}_k(\bm{x})$. Equipped by the surrogate, client $k$, now chooses the next-to-sample design by maximizing the expected utility 
\begin{align} \label{eq:single}
    \bm{x}_k^{(t)\text{new}} = \argmax_{\bm{x}} \mathbb{E}_{\mathbb{P}_{\hat{y}_k|\mathcal{D}^{(t)}_k}}\left[U(\hat{y}_k(\bm{x}); \mathcal{D}^{(t)}_k)\right].
\end{align}
Hereon, for notation brevity, we write the expected utility in \eqref{eq:single} as $\mathbb{E}_{{\hat{y}_k|\mathcal{D}^{(t)}_k}}\left[U(\hat{y}_k(\bm{x}))\right]$. It is also worth noting that expectation in \eqref{eq:single} with respect to $\hat{y}_k$ is sometimes replaced with $\hat{f}_k$ where  $\hat{y}_k=\hat{f}_k+\epsilon_k$, depending on the type of utility function.  See Sec. \ref{subsec:context} for some examples. 
% Hereon, for notation consistency and brevity: (1) From an algorithmic perspective, we let $\bm{x}_{k,N^{(1)}_k}=\bm{x}_k^{(0)\text{new}}$, where $N^{(1)}_k=N^{(0)}_k+1$. Therefore, $y_{k,N^{(1)}_k}=y_k(\bm{x}_{k,N^{(1)}_k})$ and the dataset is updated as $\mathcal{D}^{(1)}_k=\mathcal{D}^{(0)}_k\cup\{\bm{x}_{k,N^{(1)}_k},y_{k,N^{(1)}_k}\}$. (2) We write the expected utility in \eqref{eq:single} as $\mathbb{E}_{{\hat{y}_k|\mathcal{D}^{(0)}_k}}\left[U(\hat{y}_k(\bm{x}))\right]$. It is also worth noting that expectation in \eqref{eq:single} with respect to $\hat{y}_k$ is sometimes replaced with $\hat{f}_k$, depending on the type of utility function. See Sec. \ref{subsec:context} for some examples. 

%The aforementioned procedures will be iterated for $T$ times. 

%%%%%%%%%%%%%%%%%%%%%
\subsection{Collaborative BO via Consensus}
\label{subsec:collaborative}
%%%%%%%%%%%%%%%%%%%%

Now in a collaborative framework, given an arbitrary $t$, we aim to allow clients to collaboratively decide on $\{\bm{x}_k^{(t)\text{new}}\}_{k=1}^K$. A natural idea is to maximize utility across all clients. This translates to the following problem: 
\begin{align}
\label{eq:original}
\max_{\{\bm{x}_k\}_{k=1}^K}\mathbb{E}_k\left[\mathbb{E}_{\hat{y}_k|\mathcal{D}^{(t)}_k}\left[U(\hat{y}_k(\bm{x}_k))\right]\right]= \max_{\{\bm{x}_k\}_{k=1}^K}\sum_{k=1}^Kp_k\left[\mathbb{E}_{\hat{y}_k|\mathcal{D}^{(t)}_k}\left[U(\hat{y}_k(\bm{x}_k))\right]\right],
\end{align}
where $p_k$ is some weight coefficient for client $k$ with $\sum_{k=1}^Kp_k=1$. Without loss of generality, hereon we assume $p_k=\frac{1}{K}$. As shown, the key difference of (\ref{eq:single}) from (\ref{eq:original}) is taking the expectation $\mathbb{E}_{k}$ over all participants. However, in itself, (\ref{eq:original}) does not allow entities to borrow strength from each other as it can be fully decoupled across the $K$ entities. 

In order to enable collaboration and distribute experiments across the $K$ clients, we bring the notion of ``consensus'' \citep{shi2017distributed, berahas2018balancing} to BO. In the context of BO, consensus allows clients to agree (reach a consensus) on their individualized next-to-sample designs through a consensus matrix. Specifically, we modify \eqref{eq:original} to 
\begin{align}
\label{eq:consensus2}
&\bm{x}^{(t)}_k=[\bm{x}^{(t)}_{\mathcal{C}}]_k=\argmax_{\bm{x}_k}\left[\mathbb{E}_{\hat{y}_k|\mathcal{D}^{(t)}_k}\left[U(\hat{y}_k(\bm{x}_k))\right]\right]
&\quad \text{and}\quad \bm{x}_k^{(t)\text{new}} = \left[(\bm{W}^{(t)}\otimes\bm{I}_D)\bm{x}^{(t)}_{\mathcal{C}}\right]_k\,,  
\end{align}
where $[\cdot]_k$ represents the $k^{th}$ block of a vector, $\bm{x}^{(t)}_\mathcal{C} =  [\bm{x}^{(t)\top}_1,  \bm{x}^{(t)\top}_2,  \cdots, \bm{x}^{(t)\top}_K]^\top$ is concatenation of the designs across all clients, $\bm{W}^{(t)}$ is a consensus matrix of size $K\times K$,
\begin{align*}
    % \bm{x}_{\mathcal{C}} =
    % \begin{bmatrix}
    %      \bm{x}_1 \\ \bm{x}_2 \\ \vdots \\ \bm{x}_K
    % \end{bmatrix} \quad \text{and} \quad 
    \bm{W}^{(t)} =
    \begin{bmatrix}
         w^{(t)}_{11} & w^{(t)}_{12} & \cdots & w^{(t)}_{1K} \\ w^{(t)}_{21} & w^{(t)}_{22} & \cdots & w^{(t)}_{2K} \\ \vdots & \vdots & \ddots & \vdots \\ w^{(t)}_{K1} & w^{(t)}_{K2} & \cdots & w^{(t)}_{KK}
    \end{bmatrix},
\end{align*}
$\bm{I}_D$ is a $D\times D$ identity matrix, and $\otimes$ denotes the Kronecker product operation, i.e., $\bm{W}^{(t)} \otimes \bm{I}_D$ results in a matrix of size $DK \times DK$. The matrix $\bm{W}^{(t)}$ is a symmetric, doubly stochastic matrix (i.e., $\sum_k w^{(t)}_{kj} = \sum_j w^{(t)}_{kj}=1$ for $j,k \in [K]$) with non-negative elements. 

The new objective \eqref{eq:consensus2} has several interesting features. For the sake of compactness, unless necessary, we drop the superscript $t$ in the subsequent discussion. First, the formulation presented in \eqref{eq:consensus2} is indeed reminiscent of recent optimization approaches coined as consensus optimization. To see this, notice that a doubly stochastic matrix $\bm{W}$ has the property that $(\bm{W}\otimes\bm{I}_D)\bm{x}_{\mathcal{C}}=\bm{x}_{\mathcal{C}}$ if and only if $\bm{x}_k=\bm{x}_j$ for all $k,j \in [K]$~\citep{nedic2009distributed}. As such, if we enforce $(\bm{W}\otimes\bm{I}_D)\bm{x}_{\mathcal{C}}=\bm{x}_{\mathcal{C}}$ as a constraint, we can solve  
\begin{align}
\label{eq:consensus}
\max_{\bm{x}_k} \frac{1}{K} \sum_{k=1}^K\left[\mathbb{E}_{\hat{y}_k|\mathcal{D}_k}\left[U(\hat{y}_k(\bm{x}_k))\right]\right] \quad \text{subject to}\quad \bm{x}_k=\bm{x}_j, \,\,\, \forall k,j \in [K].  
\end{align}
The equality constraint (often referred to as the consensus constraint) is imposed to enforce that local copies at every client are equal.  % (under the assumption that the connectivity of the clients forms a connected network). 
That said, our formulation and setting have several distinguishing features that differentiate it from consensus optimization. The differences pertain to the goals of the two approaches and are related to the consensus matrix $\bm{W}$. In consensus optimization, the goal is for all clients to eventually agree on a common decision variable. However, in our collaborative BO paradigm, we do not want all clients to make the same decisions. Rather, we want clients to borrow strength from each other while at the same time allowing for personalized (per client) solutions. Even when clients are homogeneous, enforcing the constraint will significantly reduce exploration as everyone runs an experiment at the same location. Therefore, we do not explicitly enforce the consensus constraint. Instead, we allow $\bm{W}$ to play an aggregation role that decides on the next-to-sample designs $\bm{x}_k^{(t)\text{new}}$ given the current utility maximizers $\bm{x}_{\mathcal{C}}^{(t)}$. More importantly, $\bm{W}^{(t)}$ is time-varying and, in the limit, converges to the identity matrix (see below for more details and examples of the consensus matrices used). This allows clients to borrow strength in the initial stages of the optimization, yet, eventually, make personalized decisions.

Second, the consensus matrix $\bm{W}$ is doubly stochastic. By the Birkhoff–von Neumann theorem \citep{marshall1979inequalities}, any doubly stochastic matrix can be expressed as a convex combination of multiple permutation matrices. Mathematically, there exists $L$ non-negative scalars $\{\eta_l\}_{l=1}^L$ such that $\sum_{l=1}^L\eta_l=1$ and $\bm{W}=\sum_{l=1}^L\eta_l\bm{P}_l$. For example, 
\begin{align*}
    \bm{W}=\begin{bmatrix}0.2&0.3&0.5\\0.6&0.2&0.2\\0.2&0.5&0.3\end{bmatrix}=0.2\begin{bmatrix}0&1&0\\0&0&1\\1&0&0\end{bmatrix}+0.2\begin{bmatrix}1&0&0\\0&1&0\\0&0&1\end{bmatrix}+0.1\begin{bmatrix}0&1&0\\1&0&0\\0&0&1\end{bmatrix}+0.5\begin{bmatrix}0&0&1\\1&0&0\\0&1&0\end{bmatrix},
\end{align*}
where each permutation matrix $\bm{P}_l$ can be viewed as an allocator that assigns one design solution to one client. In this example, we can see that
\begin{align*}
    (\bm{P}_1\otimes\bm{I}_D)\bm{x}_{\mathcal{C}}=\left(\begin{bmatrix}0&1&0\\0&0&1\\1&0&0\end{bmatrix}\otimes\bm{I}_D\right)\begin{bmatrix}
         \bm{x}_1 \\ \bm{x}_2 \\ \bm{x}_3
    \end{bmatrix} =\begin{bmatrix}
         \bm{x}_2 \\ \bm{x}_3 \\ \bm{x}_1
    \end{bmatrix}.
\end{align*}
This implies that the first allocator assigns the solution from client 1 to client 2, and so forth. This is an interesting phenomenon that is related to batch BO, where a batch of candidate design points is selected, and then multiple experiments run in parallel, each corresponding to one candidate solution. Now, instead of sticking to one allocation strategy, the consensus approach can be viewed as a weighted average of all possible assignments, where the designed consensus matrix dictates the weights, hence the impact of clients on each other.  

Third, over the course of the optimization, the consensus step $(\bm{W}\otimes\bm{I}_D)\bm{x}_{\mathcal{C}}$ naturally yields $K$ solutions to all clients (i.e., one for each) so that clients can explore and exploit the solution space independently and in a distributed manner. For example, suppose $D=1$ and consider a consensus matrix $\bm{W}=\begin{bmatrix}0.7&0.3\\0.3&0.7\end{bmatrix}$ and two local solutions $\bm{x}_1=5, \bm{x}_2=7$. The consensus step will yield $\bm{W}\bm{x}_{\mathcal{C}}=(5.6,6.4)^\intercal$ such that client 1 will take solution $5.6$, and client 2 will take solution $6.4$.

Fourth, from the previous example, it can be seen that the consensus matrix $\bm{W}$ controls how much one client will affect the design choices for other clients. As a result, the matrix $\bm{W}$ adds a layer of flexibility in optimizing \eqref{eq:consensus2} and allows for heterogeneous clients. More specifically, in Sec. \ref{subsec:matrix}, we will show two approaches to design $\bm{W}^{(t)}$ at each iteration $t$ such that $\bm{W}^{(t)}\to\bm{I}$. The intuition is as follows. In the early stages, client $k$ may not have enough observations to obtain a high-quality surrogate model and therefore needs to borrow information from other clients. In the late stages, as client $k$ accumulates sufficient data and can construct a high-quality local surrogate model, it will focus more on its local design problem to find client-specific optimal design points.

Finally, the consensus constraint provides a naturally distributed approach to solve \eqref{eq:consensus2} sequentially. To see this, recall that each client can run $T$ experiments across $T$ iterations and the initial dataset is $\mathcal{D}_k^{(0)}$. Then a natural algorithm would take the following form: At iteration $t\in\{0,\ldots,T-1\}$, given $\bm{x}_{\mathcal{C}}^{(t)}$, all clients calculate $\bm{x}_k^{(t)\text{new}}=[(\bm{W}^{(t)}\otimes \bm{I}_D)\bm{x}_{\mathcal{C}}^{(t)}]_k$. Now, each client $k$ conducts an experiment at $\bm{x}_k^{(t)\text{new}}$, and augments its dataset with the new observation. Each client then fits a surrogate (e.g., $\mathcal{GP}$) using the current dataset $\mathcal{D}_k^{(t+1)}$ and accordingly constructs the utility function $U(\hat{y}_k(\bm{x}_k))$. Then the clients maximize their local utility  $\max_{\bm{x}_k}\mathbb{E}_{\hat{y}_k|\mathcal{D}^{(t+1)}_k}\left[U(\hat{y}_k(\bm{x}_k))\right]$ to obtain a new candidate solution $\bm{x}^{(t+1)}_k$. Finally, all clients send their solutions to each other or a central orchestrator, and the process repeats. 

Algorithm \ref{alg::consensus} summarizes our  Collaborative Bayesian Optimization via Consensus (\texttt{CBOC}) framework. Fig. \ref{fig:collaboration} presents a flowchart illustrating the collaborative BO. Notably, our collaborative framework enjoys some nice theoretical properties (See Sec. \ref{subsec:theory}). 

%the orchestrator distributes the corresponding component $\bm{x}_k^{(t)\text{new}}$ to each client $k$.

\vspace{0.5cm}
\begin{algorithm}[H]
	\caption{\texttt{CBOC}: Collaborative BO via Consensus}
	\label{alg::consensus}
		\SetAlgoLined
	\KwData{$T$ iterations, $K$ number of clients, initial data $\{\mathcal{D}^{(0)}_{k}\}_{k=1}^K$, initial consensus matrix $\bm{W}^{(0)}$, initial designs to share $\{\bm{x}_k^{(0)}\}_{k=1}^K$}
		\For{$t$ $= 0, 1, \cdots T-1$}{
        \For{$k=1,\cdots,K$}{
            {\bfseries Consensus:} client $k$ computes the consensus point $\bm{x}_k^{(t)\text{new}}=[(\bm{W}^{(t)}\otimes\bm{I}_D)\bm{x}_{\mathcal{C}}^{(t)}]_k$ \;
		{\bfseries Experiment:} client $k$ conducts experiment using $\bm{x}_k^{(t)\text{new}}$ and observes $y_k(\bm{x}_k^{(t)\text{new}})$\;
            {\bfseries Data Augmentation:} client $k$ augments dataset by $\mathcal{D}_k^{(t+1)}=\mathcal{D}_k^{(t)}\cup\{\bm{x}_k^{(t)\text{new}},y_k(\bm{x}_k^{(t)\text{new}})\}$ \;
            {\bfseries Surrogate Modeling:} client $k$ updates its surrogate model \;
            {\bfseries Optimization:} client $k$ finds their local utility maximizer $\bm{x}^{(t+1)}_k=\argmax_{\bm{x}_k}\mathbb{E}_{\hat{y}_k|\mathcal{D}^{(t+1)}_k}\left[U(\hat{y}_k(\bm{x}_k))\right]$ and shares this candidate sample \;
            }
}
% Return  $\bm{x}^{(T)\text{new}}=(\bm{W}^{(T)}\otimes\bm{I}_D)\bm{x}_{\mathcal{C}}^{(T)}$
\end{algorithm}
\vspace{0.5cm}

\begin{figure}[!htbp]
    \centering
    \centerline{\includegraphics[width=0.8\columnwidth]{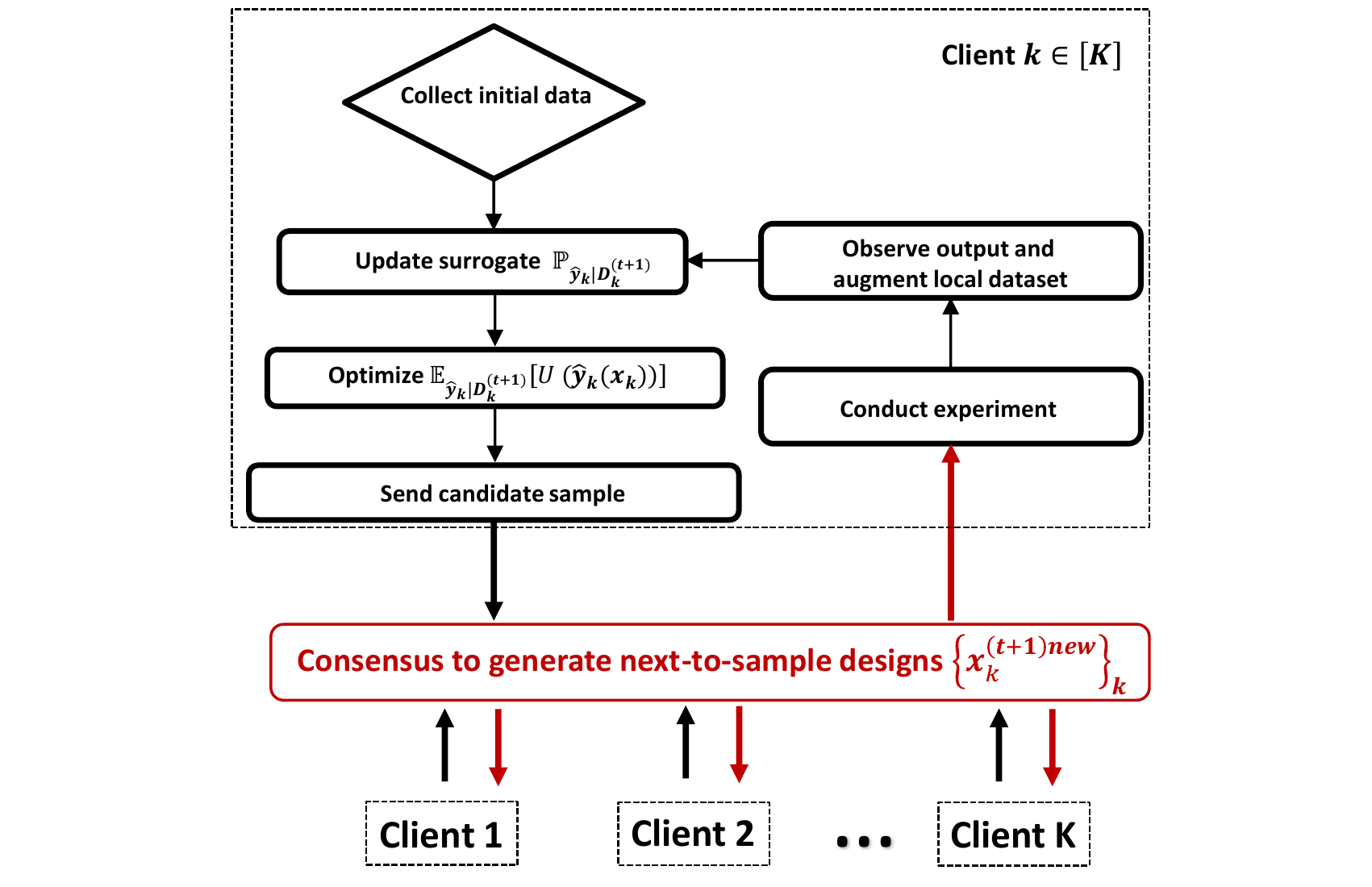}}
    \caption{Flowchart of CBOC.}
    \label{fig:collaboration}
\end{figure}

It is worthwhile noting that the collaborative process can be done in a centralized or decentralized manner where clients share $\bm{x}^{(t)}_k$ either with a central orchestrator or with each other directly. 

\subsection{Designing the Consensus Matrix}
\label{subsec:matrix}

From Algorithm \ref{alg::consensus}, it can be seen that one of the key components in \eqref{eq:consensus2} is the consensus matrix $\bm{W}$. In essence, as we have discussed earlier, the consensus matrix $\bm{W}$ controls how much one client will affect the design choices for other clients. Therefore, one needs to carefully design the consensus matrix.

In this section, we propose two approaches for designing the consensus matrix. The first approach assumes one does not have any prior information on different clients and uniformly adjusts all entries in $\bm{W}^{(t)}$ at each iteration $t$. The second approach carefully modifies the weights for each client based on the ``leader", where the ``leader'' is defined as the client that has observed the best improvement (e.g., the most significant improvement in the utility). Below, we will detail both approaches.

% \begin{align*}
%     \bm{W}^{(1)}=\begin{bmatrix}\frac{1}{K}+\frac{K-1}{TK}&\frac{1}{K}-\frac{1}{TK}&\ldots&\frac{1}{K}-\frac{1}{TK}\\ \vdots & \vdots & \vdots & \vdots \\ \frac{1}{K}-\frac{1}{TK}&\frac{1}{K}-\frac{1}{TK}&\ldots&\frac{1}{K}+\frac{K-1}{TK}\end{bmatrix}.
% \end{align*}
% \begin{align*}
%     \bm{W}^{(T-1)}=\begin{bmatrix}1&0&\ldots&0\\ \vdots & \vdots & \vdots & \vdots \\ 0&0&\ldots&1\end{bmatrix}.
% \end{align*}

\paragraph{Uniform Transitional Matrix} We propose to assign equal weights to all clients at iteration $0$ and then gradually decay off-diagonal elements
\begin{align*}
    \bm{W}^{(0)}=\begin{bmatrix}\frac{1}{K}&\frac{1}{K}&\ldots&\frac{1}{K}\\ \vdots & \vdots & \vdots & \vdots \\ \frac{1}{K}&\frac{1}{K}&\ldots&\frac{1}{K}\end{bmatrix}.
\end{align*}
For simplicity, we assume all clients are connected. In practice, if clients $i$ and $j$ are not connected, we can set $w_{ij}=w_{ji}=0$ and reweigh the other components of the $\bm{W}$ matrix to ensure it is doubly stochastic.

\begin{figure}[!htbp]
    \centering
    \centerline{\includegraphics[width=0.6\columnwidth]{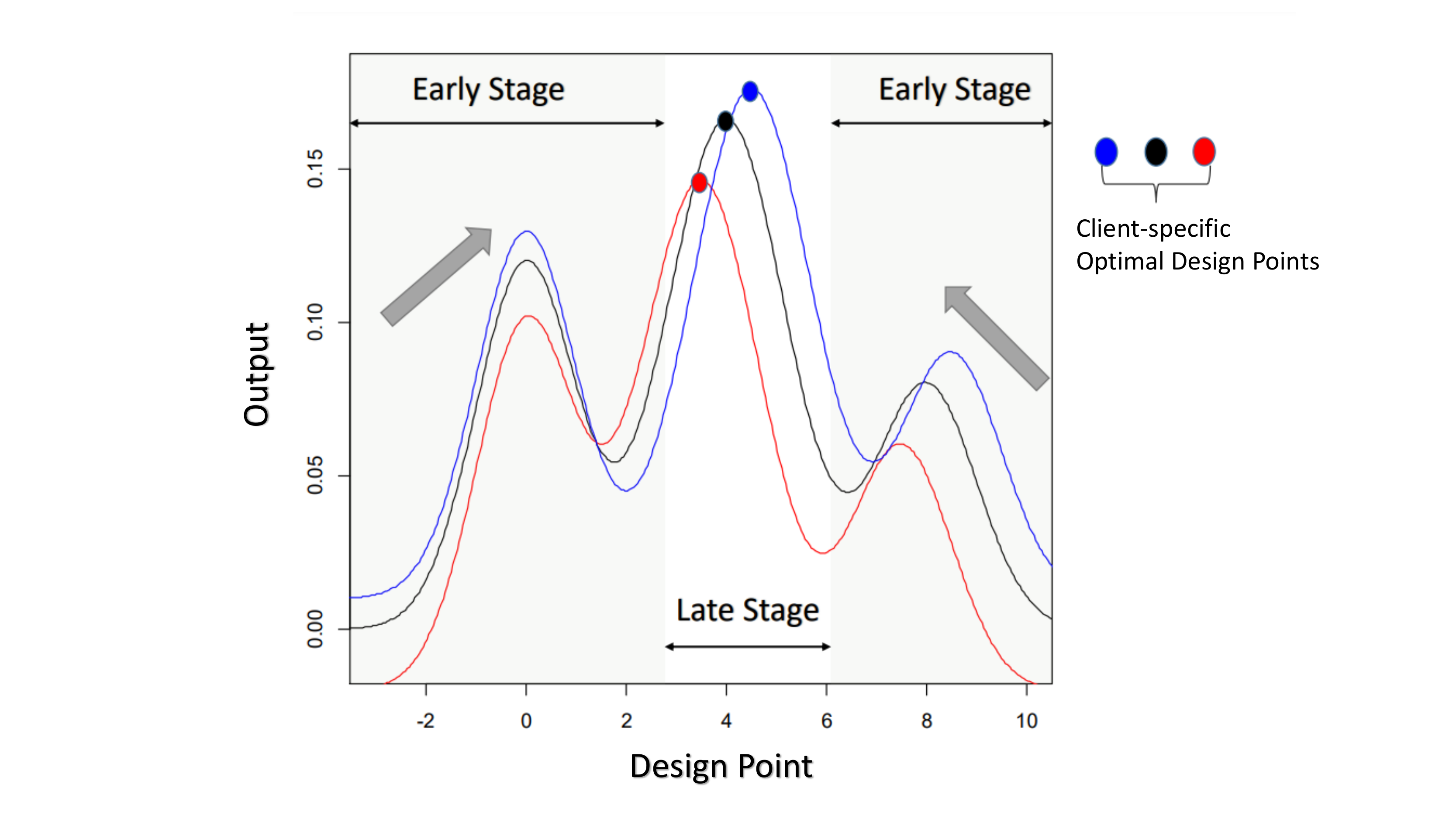}}
    \caption{Illustration of the collaborative process. Colors represent client-specific black-box functions.}
    \label{fig:change}
\end{figure}

As the optimization progresses, we gradually increase diagonal elements $w_{kk}$ to 1 and decrease the off-diagonal elements $w_{kj}$ for $k \neq j$ to 0. For example, if we adjust weights linearly, then
\begin{align}
\label{eq:W1}
    \bm{W}^{(t+1)}=\bm{W}^{(t)}+\begin{bmatrix}\frac{K-1}{TK}&-\frac{1}{TK}&\ldots&-\frac{1}{TK}\\ \vdots & \vdots & \vdots & \vdots \\ -\frac{1}{TK}&-\frac{1}{TK}&\ldots&\frac{K-1}{TK}\end{bmatrix}.
\end{align}
Ultimately, $\bm{W}^{(t)}$ will converge to the identity matrix $\bm{I}$. The intuition behind this design is illustrated in Fig. \ref{fig:change}. In this figure, each colored line represents a black-box design function for a client. In the early stages, client $k$ does not have enough observations to obtain a high-quality utility and therefore needs to borrow more information from other clients. As client $k$ has more data in the late stages, it will focus more on its own design problem to find client-specific optimal designs.

Note that $\bm{W}^{(t)} \rightarrow \bm{I}$, is mandatory if clients have some heterogeneity. To see this, assume all clients have recovered their optimal design $\bm{x}_k^*$, then if $\bm{W} \neq \bm{I}$, the consensus step will always move the experiment away from $\bm{x}_k^*$.

% Here note that the identity matrix at late stages is mandatory for convergence. Otherwise, consider a scenario where $K=2$, $\bm{x}_1^*=10$, and $\bm{x}_2^*=20$. Assume, at iteration $T$, both clients recover the optimal designs but $\bm{W}^{(T)}=\begin{bmatrix}0.5&0.5\\0.5&0.5 \end{bmatrix}$. As a result, one will obtain $\bm{x}_1^{(T)\text{new}}=\bm{x}_2^{(T)\text{new}}=15$. This will prevent collaboration from ever converging.

\paragraph{Leader-driven Matrix} Another approach is to adjust the weights of $\bm{W}$ dynamically based on the ``leader", i.e., the client that observed the best improvement. As a result, all other clients will follow the leader and explore the region that potentially contains the best solutions. Initially, we create two matrices $\bm{W}^{(0)}_1$ and $\bm{W}^{(0)}_2$, and assign equal weights to all entries of both matrices. Matrix $\bm{W}_1$ will be used as a baseline, and is updated using \eqref{eq:W1}. Matrix $\bm{W}_2$ will be modified based on $\bm{W}_1$ and by incorporating the leader information. The purpose of using $\bm{W}_1$ is to ensure that $\bm{W}_2$ still converges to $\bm{I}$ so that every single client will eventually focus on their own design objective. 

At iteration $t$, each client $k$ shares a pair $(\bm{x}^{(t)}_k,S^{(t)}_k)$, where $S_k$ is a reward that quantifies the gained benefit at client $k$. As a simple example, we could set $S^{(t)}_k=\max\mathbb{E}_{\hat{y}_k|\mathcal{D}^{(t)}_k}\left[U(\hat{y}_k(\bm{x}_k))\right]$. We then sort $S^{(t)}_k$ and find $k^{*(t)}=\argmax_k\{S^{(t)}_1,S^{(t)}_2,\ldots,S^{(t)}_K\}$, and treat client $k^{*(t)}$ as a leader at the current iteration. In essence, client $k^{*(t)}$ can hint to others that the neighborhood of $\bm{x}_{k^{*(t)}}^{(t)}$ is an area of potential improvement. 
To do so, during the consensus step, for the non-lead clients $k\neq k^{*(t)}$, we decrease their weights by $\frac{1}{TK}$ and increase their off-diagonal weights with the leader $k^{*(t)}$ by $\frac{K-1}{TK}$. To this end, we adjust blocks of $\bm{W}^{(t)}_2$ as follows: 
\begin{align*}
[\bm{W}^{(t)}_2]_{\neq k^{*(t)}, k^{*(t)}} = [\bm{W}^{(t)}_1]_{\neq k^{*(t)}, k^{*(t)}}+\frac{K-1}{TK},\\ 
[\bm{W}^{(t)}_2]_{k^{*(t)},\neq k^{*(t)}} = [\bm{W}^{(t)}_1]_{k^{*(t)},\neq k^{*(t)}}+\frac{K-1}{TK},\\
[\bm{W}^{(t)}_2]_{\neq k^{*(t)}, \neq k^{*(t)}} = [\bm{W}^{(t)}_1]_{\neq k^{*(t)}, \neq k^{*(t)}} - \frac{1}{TK},\\
[\bm{W}^{(t)}_2]_{k^{*(t)}, k^{*(t)}} = [\bm{W}^{(t)}_1]_{k^{*(t)}, k^{*(t)}} - \frac{(K-1)^2}{TK}, 
\end{align*}
where $[\bm{W}_2^{(t)}]_{\neq i,j}$ represents all elements that are in the $j^{th}$ column but not in the $i^{th}$ row of $\bm{W}_2^{(t)}$.
% the $(i,j)^{th}$ component in $\bm{W}_2^{(t)}$. 
Interestingly, the weight $[\bm{W}^{(t)}_2]_{k^{*(t)}, k^{*(t)}}$ shrinks to maintain double stochasticity. Such an assignment allows the leader $k^{*(t)}$ to explore new regions. 

Here we also suggest two heuristics: (i) To avoid the situation where the same client is selected in succession (this will cause the same client to keep exploring rather than exploiting), we propose to select the second largest index from $\{S^{(t+1)}_1,S^{(t+1)}_2,\ldots,S^{(t+1)}_K\}$ if $k^{*(t)}=k^{*(t+1)}$. (ii) In the extreme case where  $[\bm{W}^{(t)}_2]_{k^{*(t)}, k^{*(t)}}$ is negative, set $[\bm{W}^{(t)}_2]_{k^{*(t)}, k^{*(t)}}$ to zero and reweigh other components accordingly. 

To see an example, consider a scenario where $D=1$, $K=3$ and $T=10$. We initialize both $\bm{W}^{(0)}_1$ and $\bm{W}^{(0)}_2$ as
\begin{align*}
    \bm{W}^{(0)}_1=\bm{W}^{(0)}_2=
    \begin{bmatrix}
    \frac{1}{3}&\frac{1}{3}&\frac{1}{3}\\ \frac{1}{3}&\frac{1}{3}&\frac{1}{3}\\ \frac{1}{3}&\frac{1}{3}&\frac{1}{3}\end{bmatrix}.
\end{align*}
Suppose each client received $(\bm{x}^{(0)}_1,1),(\bm{x}^{(0)}_2,5),(\bm{x}^{(0)}_3,4)$. Then we adjust $\bm{W}^{(0)}_2$ as
\begin{align*}
    \bm{W}^{(0)}_2=
    \begin{bmatrix}
    \frac{1}{3}-\frac{1}{30}&\frac{1}{3}+\frac{2}{30}&\frac{1}{3}-\frac{1}{30}\\ 
    \frac{1}{3}+\frac{2}{30}&\frac{1}{3}-\frac{4}{30}&\frac{1}{3}+\frac{2}{30}\\ \frac{1}{3}-\frac{1}{30}&\frac{1}{3}+\frac{2}{30}&\frac{1}{3}-\frac{1}{30}\end{bmatrix}.
\end{align*}
In this example, client 2 observed that sampling at $\bm{x}^{(0)}_2$ yields the highest benefits, and therefore we put more weight on client 2. As a result, clients 1 and 3 will explore towards $\bm{x}^{(0)}_2$. On the other hand, we decrease $[\bm{W}_2^{(0)}]_{2,2}$ to make sure that client 2 does not over-explore the region that contains $\bm{x}^{(0)}_2$.

\subsection{Contextualization under a specific surrogate and utility}
\label{subsec:context}

Now, given the generic framework presented in \ref{subsec:collaborative} and \ref{subsec:matrix}, we will contextualize Algorithm \ref{alg::consensus} and provide a concrete example. The iteration index $t$ is dropped for simplicity unless necessary. 

\subsubsection{Gaussian Process Surrogate}

% A $\mathcal{GP}$ is characterized by its mean function $m_k(\bm{x})=\mathbb{E}[\hat{f}_k(\bm{x})]$ and kernel function $\mathcal{K}_k(\bm{x},\bm{x}')=\mathbb{E}[(\hat{f}_k(\bm{x})-m_k(\bm{x}))(\hat{f}_k(\bm{x}')-m_k(\bm{x}'))]$ with kernel parameter $\bm{\theta}_{k}$. Without loss of generality, in this work, we assume $m_k(\bm{x})\equiv0,\forall \bm{x}\in\mathbb{R}^D.$

At each iteration, we place a $\mathcal{GP}$ prior on the surrogate $\hat{f}_k$.
% We assume  additive noise term $\epsilon_k$, where 
% We assume that the data-generating process is  $y_{k,i}=f_k(\bm{x}_{k,i})+\epsilon_{k,i}$, $\forall i\in[N_k], k \in[K]$, where $\epsilon_{k,i}$ is an additive noise term. 
We model the additive noise term $\epsilon_{k}$ as independent and identically distributed ($i.i.d.$) noise that follows a normal distribution with zero mean and $v_k^2$ variance.  Now, given a new input location $\bm{x}^{\text{test}}$, the posterior predictive distribution of $\hat{f}_k(\bm{x}^{\text{test}}) \sim \mathbb{P}_{\hat{f}_k(\bm{x}^{\text{test}})|\mathcal{D}^{(t)}_k} \coloneqq \mathcal{N}(\mu_{k}(\bm{x}^{\text{test}};\mathcal{D}_k),\sigma^2_{k}(\bm{x}^{\text{test}};\mathcal{D}_k))$, where \begin{equation}
\begin{split}
\label{eq:pred}
    &\mu_{k}(\bm{x}^{\text{test}};\mathcal{D}_k)=\bm{K}(\bm{x}^{\text{test}}, \bm{X}_k)\left(\bm{K}(\bm{X}_k, \bm{X}_k)+v_k^2\bm{I}\right)^{-1}\bm{y}_k,\\
    &\sigma^2_{k}(\bm{x}^{\text{test}};\mathcal{D}_k)=\bm{K}(\bm{x}^{\text{test}}, \bm{x}^{\text{test}})-\bm{K}(\bm{x}^{\text{test}}, \bm{X}_k)\left(\bm{K}(\bm{X}_k,\bm{X}_k)+v_k^2\bm{I}\right)^{-1}\bm{K}(\bm{X}_k,\bm{x}^{\text{test}}),
\end{split}
\end{equation}
and $\bm{K}(\cdot,\cdot):\mathbb{R}^D\times\mathbb{R}^D\to\mathbb{R}$ is a covariance matrix whose entries are determined by some kernel function $\mathcal{K}(\cdot,\cdot)$. Similarly, $\mathbb{P}_{\hat{y}_k|\mathcal{D}^{(t)}_k}$ is derived by simply adding $v_k^2$ to $\sigma^2_{k}(\bm{x}^{\text{test}})$.

\subsubsection{Utility Function}

Given the $\mathcal{GP}$ surrogate, one can build a utility function that measures the benefits of conducting an experiment using a set of new design points. One common example is the expected improvement (\texttt{EI}) utility expressed as \citep{jones1998efficient}
\begin{align*}
    \mathbb{E}_{\hat{f}_k|\mathcal{D}^{(t)}_k}\left[U(\hat{f}_k(\bm{x}))\right]&=\mathbb{E}_{\hat{f}_k|\mathcal{D}^{(t)}_k}\left[(\hat{f}_k(\bm{x})-y_k^{*(t)})^+\right]=\texttt{EI}^{(t)}_k(\bm{x})\\
    &=\sigma^{(t)}_k(\bm{x};\mathcal{D}^{(t)}_k)\phi(z^{(t)}_k(\bm{x}))+(\mu^{(t)}_k(\bm{x};\mathcal{D}^{(t)}_k)-y^{*(t)}_{k})\Phi(z^{(t)}_k(\bm{x})),
\end{align*}
where $a^+=\max(a,0)$, $y_k^{*(t)}=\max \bm{y}^{(t)}_k$ is the current best response,  $\phi(\cdot)$ (or $\Phi(\cdot)$) is a probability density function (PDF) (or cumulative distribution function (CDF)) of a standard normal random variable, and $z^{(t)}_k(\bm{x})=\frac{\mu^{(t)}_k(\bm{x};\mathcal{D}^{(t)}_k)-y^{*(t)}_{k}}{\sigma^{(t)}_k(\bm{x};\mathcal{D}^{(t)}_k)}$. Here note that the expectation is taken with respect to $\hat{f}_k$ rather than $\hat{y}_k$. Another example is the knowledge gradient (\texttt{KG}) \citep{wu2016parallel} defined as $\mathbb{E}_{\hat{y}_k|\mathcal{D}^{(t)}_k}\left[U(\hat{y}_k(\bm{x}))\right]=\texttt{KG}^{(t)}_k(\bm{x})
    =\mathbb{E}_{\hat{y}_k|\mathcal{D}^{(t)}_k}\left[ \mu^{*{(t)}}_k(\bm{x};\mathcal{D}^{(t)}_k\cup\{\bm{x},\hat{y}_k(\bm{x})\}) \right] - \max_{\bm{x}'}(\mu^{(t)}_k(\bm{x}';\mathcal{D}^{(t)}_k))$,
where $\mu^{*(t)}_k(\bm{x};\mathcal{D}^{(t)}_k\cup\{\bm{x},\hat{y}_k(\bm{x})\})$ is the maximum of the updated posterior mean of the $\mathcal{GP}$ surrogate by taking one more sample at location $(\bm{x},\hat{y}_k(\bm{x}))$. \texttt{KG} can be interpreted as finding the new sampling location $\bm{x}$ that potentially increases the maximum updated posterior mean. Hereon, in the remainder of this paper, we drop $\mathcal{D}_k$ in $\mu_k,\sigma_k$ for the sake of compactness.

\subsection{Theoretical Analysis}
\label{subsec:theory}

% Our collaborative framework \texttt{CSDC} is endowed with favorable properties. 

Despite its immense success, BO theory is still in its infancy due to many fundamental challenges. First and foremost, the black-box nature of $f_k$ renders theory hard to derive due to the lack of known structure. Second, we still have a limited understanding of the properties, such as Lipschitz continuity, concavity, or smoothness, of many commonly employed utility functions. For example, even the \texttt{EI} utility in general is not Lipschitz continuous or concave. Third, despite recent advances in laying the theoretical foundations for understanding the generalization error bounds of $\mathcal{GP}$s \citep{lederer2019uniform, wang2020prediction}, bridging the gap between these bounds and errors incurred in the utility function remains an open problem and a rather challenging one. 

To circumvent these open problems while providing a theoretical proof of concept, we derive some theoretical insights that are confined to the \texttt{EI} utility in conjunction with a smooth $\mathcal{GP}$ kernel and a homogeneity assumption. While we believe that our results extend beyond these settings, we leave this analysis as an enticing challenge for future research. 

We focus on regret defined as  $r^{(t)}_{k}=f_{k}(\bm{x}_k^*)-f_{k}(\bm{x}^{(t)\text{new}}_k)$ for client $k$ at iteration $t$. Intuitively, regret measures the gap between the design function evaluated at the optimal solution $\bm{x}_k^*$ and the one evaluated at the consensus solution $\bm{x}^{(t)\text{new}}_k$, at iteration $t$. By definition, $r^{(t)}_{k}=0$ if \texttt{CBOC} recovers the global optimal solution. Our theoretical guarantee studies cumulative regret, more specifically, $R_{k,T}=\sum_{t=1}^Tr^{(t)}_{k}$.

% It is challenging to present theoretical results for collaborative sequential design algorithms due to heterogeneity and the unknown nature of the structure and properties of the design function $f_k$. Furthermore, several commonly employed utility functions lack desirable properties like concavity and Lipschitz continuity. In fact, the existing body of literature exploring this particular aspect is quite limited, and very few papers have dived into this domain. 

% In this section, we first provide a theoretical guarantee of our proposed framework under the homogeneous setting (i.e., $f_1=f_2=\cdots=f_k$) and then briefly discuss the heterogeneous scenario. 

%We define a stopping criterion for the \texttt{EI} utility function $\texttt{EI}_k^{(t)}$. 
By definition, $\texttt{EI}_k^{(t)}(\bm{x})=\mathbb{E}_{\hat{f}_k|\mathcal{D}^{(t)}_k}\left[(\hat{f}_k(\bm{x})-y_k^{*(t)})^+\right]\geq 0, \forall \bm{x}\in\mathbb{R}^D$. Therefore, we define a small positive constant $\kappa$ such that when $\texttt{EI}_k^{(t)}(\bm{x})<\kappa$, client $k$ stops its algorithm at iteration $t$. This stopping criterion is only used for theoretical development. In practice, we will run our algorithm for $T$ iterations or until all budgets are exhausted.

Below, we present our main Theorem and the sketch of our proof. Detailed information and all supporting Lemmas and required assumptions can be found in Appendices 1-2. 

% \begin{theorem}
% \label{theorem:1}
% Suppose a squared exponential kernel function $\mathcal{K}_k(\bm{x}_1,\bm{x}_2)=u_{k}^2\exp\left(\frac{\norm{\bm{x}_1-\bm{x}_2}^2}{2\ell_k^2}\right)$ is used for the $\mathcal{GP}$ surrogate and each client uses the {\normalfont\texttt{EI}} utility. Under some assumptions (Appendix A), with probability at least $1-\delta$, the cumulative regret after $T$ iterations is $R_{k,T}=\sum_{t=1}^Tr^{(t)}_{k}\sim\mathcal{O}(\sqrt{T\times(\log T)^{D+4}})$. Furthermore, $\lim_{T\to\infty}\frac{R_{k,T}}{T}=0$.
% \end{theorem}

\begin{theorem}
\label{theorem:1}
(Homogeneous Clients) Suppose $f_1=f_2=\cdots=f_K$, and suppose a squared exponential kernel function $\mathcal{K}_k(\bm{x},\bm{x}')=u_{k}^2\exp\left(\frac{\norm{\bm{x}-\bm{x}'}^2}{2\ell_k^2}\right)$ is used for the $\mathcal{GP}$ surrogate, where $u_k$ is the variance scale parameter and $\ell_k$ is the length parameter, and each client uses the {\normalfont\texttt{EI}} utility. Without loss of generality, assume $\mathcal{K}_k(\bm{x},\bm{x}')\leq 1$ and the initial sample size for each client is $2$. Under some assumptions (Appendix 1), given any doubly stochastic $\bm{W}^{(t)}$ with non-negative elements, for $\epsilon>0, \delta_1\in(0,T)$, with probability at least $(1-\frac{\delta_1}{T})^T$, the cumulative regret after $T>1$ iterations is 
\begin{align*}
    R_{k,T}&=\sum_{t=1}^Tr_{k}^{(t)}\leq \sqrt{\frac{6T\left[(\log T)^3+1+C \right](\log T)^{D+1}}{\log(1+v_k^{-2})}}\\
    &\qquad +\sqrt{\frac{2T(\log T)^{D+4}}{\log(1+v_k^{-2})}}+\sum_{t=1}^T\mathcal{O}\left(\frac{1}{(\log (2+t))^{0.5+\epsilon}}\right)\\
    &\sim\mathcal{O}(\sqrt{T\times(\log T)^{D+4}}),
\end{align*}
where $C=\log[\frac{1}{2\pi\kappa^2}]$.
\end{theorem}

Theorem \ref{theorem:1} shows that the cumulative regret of Algorithm \ref{alg::consensus} has a sublinear growth rate in terms of the number of iterations. This implies that as the algorithm proceeds with more iterations, the consensus solution for client $k$ will be close to the optimal solution $\bm{x}_1^*=\cdots=\bm{x}_k^*$.

Theorem \ref{theorem:1} relies on the assumption of homogeneity. When heterogeneity is present, where at least one pair of design functions has some differences, providing convergence results is challenging.  To our knowledge, the theory here is vacant. Even if one were to provide a regularity assumption on the heterogeneity, say $f_k$'s are uniformly bounded and have optimal designs restricted within a small ball, it remains an open problem to understand how this regularity will propagate to the utility where we are able to understand the structural similarities and differences across $\{\texttt{EI}_k\}_{k=1}^K$. Indeed, this stems from the aforementioned second and third challenges in traditional BO. That said, Theorem \ref{theorem:1} may serve as a proof of concept, showcasing that our approach will drive clients to a region that contains optimal designs. Afterward, each client will focus on their own objective to obtain a client-specific optimal design (recall Fig. \ref{fig:change}). As shown in Sec. 4.3 and Sections 5-6, our empirical results highlight the advantageous properties of our approach under heterogeneous settings.

\paragraph{Proof Sketch}

% \paragraph{Proof Sketch} 
We hereby provide a proof sketch for Theorem \ref{theorem:1}. We first decompose the regret $r_{k}^{(t)}$ as
 \begin{align*}
    r_{k}^{(t)}=\bigg\{\underbrace{f_{k}(\bm{x}_k^*)- y^{*(t)}_{k} }_{\text{A}} + \underbrace{y^{*(t)}_{k}- f_{k}(\bm{x}^{(t)\text{new}}_k)}_{\text{B}}\bigg\}.
\end{align*}
We let $\text{A}=f_{k}(\bm{x}_k^*)- y^{*(t)}_{k}$ and $\text{B}=y^{*(t)}_{k}- f_{k}(\bm{x}^{(t)\text{new}}_k)$. Next, we aim to bound terms A and B. 

\begin{itemize}
    \item Term A defines the difference between the optimal value of the design function and the current best-observed output value. Here, we first use a concentration inequality from \cite{srinivas2009gaussian} to bound the difference between the mean function $\mu_k^{(t)}(\cdot)$ of the $\mathcal{GP}$ surrogate and the truth $f_{k}(\cdot)$. We then show that, with a high probability, term A is upper bounded by the \texttt{EI} value evaluated at the optimal design $\bm{x}_k^*$ plus a scaled predictive standard deviation term. Mathematically, with probability $1-\delta$, $\delta\in(0,1)$, $$f_{k}(\bm{x}_k^*) -y^{*(t)}_{k}\leq \texttt{EI}_k^{(t)}(\bm{x}_k^*)+\sqrt{\beta_k^{(t)}}\sigma_k^{(t)}(\bm{x}_k^*),$$
    where $\{\beta_k^{(t)}\}_t$ is a non-decreasing sequence such that $\beta_k^{(t)}\sim\mathcal{O}((\log \frac{t}{\delta})^3)$.
    \item Term B defines the difference between the current best-observed output value and the underlying true function $f_k(\cdot)$ evaluated at the consensus solution. To proceed, we further expand B as
    \begin{align*}
        \text{B}&=y^{*(t)}_{k}- f_{k}(\bm{x}^{(t)\text{new}}_k)=y^{*(t)}_{k}-\mu_k^{(t)}(\bm{x}^{(t)\text{new}}_k) +\mu_k^{(t)}(\bm{x}^{(t)\text{new}}_k)- f_{k}(\bm{x}^{(t)\text{new}}_k).
    \end{align*}
    We then bound B using a similar strategy in bounding term A and obtain
    \begin{align*}
        \text{B}\leq\sigma^{(t)}_{k}(\bm{x}^{(t)\text{new}}_k)\left(\tau(-z_k^{(t)}(\bm{x}^{(t)\text{new}}_k))-\tau(z_k^{(t)}(\bm{x}^{(t)\text{new}}_k))+\sqrt{\beta_k^{(t)}} \right),
    \end{align*}
    where $\tau(z_k^{(t)}(\bm{x}))\coloneqq z_k^{(t)}(\bm{x})\Phi(z_k^{(t)}(\bm{x}))+\phi(z_k^{(t)}(\bm{x}))$.
\end{itemize}
% After some algebra manipulation, the aforementioned two steps lead to
% \begin{align*}
%     r_{k}^{(t)}&=\text{A}+\text{B}\\
%     &\leq\sigma^{(t-1)}_{k}(\bm{x}^{(t)\text{new}}_k)\left(\tau(-z_k^{(t-1)}(\bm{x}^{(t)\text{new}}_k))+\sqrt{\beta_k^{(t)}} \right) +\sqrt{\beta^{(t)}_{k}}\sigma^{(t-1)}_{k}(\bm{x}_k^*)+M\frac{1}{(\log N_k)^\epsilon}.
% \end{align*}
Our next goal is to show that the summation of $r_{k}^{(t)}$ over $T$ iterations (i.e., the cumulative regret) is bounded. Here, note that there are two key components that appear in A and B: $\sigma^{(t)}_{k}(\cdot)$ and $\tau(\cdot)$. First, we analyze the behavior of the cumulative predictive variance $\sum_{t=1}^T\sigma^{2(t)}_{k}(\cdot)$. We show that, with a squared exponential kernel, $\sum_{t=1}^T\sigma^{2(t)}_{k}(\bm{x}^{(t)\text{new}}_k)\leq \frac{2}{\log(1+v_k^{-2})}\mathcal{O}((\log T)^{D+1})$. Second, we show that $\tau(-z_k^{(t)}(\bm{x}^{(t)\text{new}}_k))\leq 1+\sqrt{C}$, where $C=\log[\frac{1}{2\pi\kappa^2}]$. After some algebraic manipulation, we obtain the upper bound stated in Theorem \ref{theorem:1}.

 \qed

\section{Simulation Studies}
\label{sec:exp}

In this section, we validate \texttt{CBOC} on several simulation datasets. We consider a range of simulation functions from the Virtual Library of Simulation Experiments \citep{surjanovic2013virtual} as the underlying black-box design functions for the clients. %\href{https://www.sfu.ca/~ssurjano/}{[link]}. 
Our goal is to showcase the benefit of collaboration in finding client-specific optimal designs.

We set the number of iterations for each testing function to $T=20D$. At iteration $0$, $\bm{W}^{(0)}$ is initialized as a uniform matrix where each entry has a weight equal to $\frac{1}{K}$. We use the uniform transitional approach (\texttt{CBOC-U}) or the leader-driven approach (\texttt{CBOC-L}) to adjust $\bm{W}^{(t)}$ at every iteration $t$. The initial dataset for each client $\{\mathcal{D}^{(0)}_{k}\}_{k=1}^K$ contains $5D$ randomly chosen design points. The performance of each client is evaluated using a Gap metric $G_k$ \citep{jiang2020binoculars} defined as
\begin{align*}
    G_k=\frac{|y^{*(0)}_k-y^{*(T)}_k|}{|y^{*(0)}_k-y^*_{k}|},
\end{align*}
where $y^{*(0)}_k$ (or $y^{*(T)}_k$) is the best observed response at iteration $0$ (or $T$), and $y^*_{k} = f_k(\bm{x}^{*}_k)$ is the true optimal response. A larger $G_k$ implies better performance. If the optimal solution is recovered (i.e., $y^{*(T)}_k=y^*_{k}$), then $G_k=1$. We compare \texttt{CBOC} using the \texttt{EI} utility with the following benchmark models: (1) \texttt{Individual}: each client $k$ conducts BO without collaboration; (2) \texttt{FedBO}: the state-of-the-art federated BO algorithm that builds upon Thompson sampling \citep{dai2021differentially}. For both benchmarks, we collect $5D$ initial samples and conduct experiments for $20D$ iterations. Our code is available at this {\href{https://github.com/UMDataScienceLab/Consensus_Bayesian_Opt/tree/main}{GitHub link}}. % To suit practitioners' needs, we use the \texttt{EI} utility coded in the BoTorch \citep{balandat2020botorch} library in Python, and users can change \texttt{EI} to any other utility functions coded in this library. 

\subsection{Testing Function 1: Levy-$D$} 
\label{subsec:test_1}

Levy-$D$ is an $D$-dimensional function in the form of
\begin{align*}
    f(\bm{x})=\sin^2(\pi \omega_1)+\sum_{d=1}^{D-1}(\omega_d-1)^2[1+10\sin^2(\pi \omega_d+1)]+(\omega_D-1)^2[1+\sin^2(2\pi \omega_D)],
\end{align*}
where $\omega_d=1+\frac{\tensor[^d]{x}{}-1}{4}, \forall d\in[D]$ and $\bm{x}=(\tensor[^1]{x}{},\ldots,\tensor[^d]{x}{}\ldots,\tensor[^D]{x}{})^\intercal\in[-10,10]^D$. We first consider the homogeneous scenario where $f_1(\bm{x})=\cdots=f_K(\bm{x})=f(\bm{x})$. We focus on maximizing the design function $-f_k(\bm{x})$ (i.e., minimizing $f_k(\bm{x})$), $\forall k\in[K]$. In Table \ref{Table:1}, we report the average Gap across $K=5$ clients over 30 independent runs, defined as follows:
\begin{align*}
    \bar{G}=\frac{1}{30}\sum_{i=1}^{30}\frac{1}{K}\sum_{k=1}^KG^{(i)}_k,
\end{align*}
where $G^{(i)}_k$ is the Gap metric for client $k$ at the $i$-th run.

% 0.866 0.764

\begin{table}[!htbp]
\centering
\begin{tabular}{cccc}
\hline
\textbf{Functions}  & \texttt{CBOC-L} & \texttt{Individual} & \texttt{FedBO}   \\ \hline
\textbf{Levy-$2$} & \bm{$0.993 (\pm 0.002)$} & $0.931(\pm 0.004)$  & $0.990(\pm 0.002)$ \\ \hline
\textbf{Levy-$4$} & \bm{$0.987(\pm 0.003)$} & $0.926(\pm 0.007)$  & $0.951(\pm 0.006)$   \\ \hline
\textbf{Levy-$8$} & \bm{$0.969(\pm 0.010)$} & $0.925(\pm 0.013)$  & $0.938(\pm 0.009)$  \\ \hline
\end{tabular}
\caption{The average Gap across $K=5$ clients over 30 independent runs under a homogeneous setting. We report standard deviations over 30 runs in brackets.}
\label{Table:1}
\end{table}

Second, we consider the heterogeneous scenario where each client has a different underlying truth. To do so, for each client $k$, we transform the Levy function $f(\bm{x})$ to
\begin{align*}
    f_k(\bm{x})=a_1f(\bm{x}+\text{vec}(a_3))+a_2 \, ,
\end{align*}
where $a_1\sim\texttt{Uniform}(0.5, 1)$, $a_2,a_3\sim\mathcal{N}(0,1)$, and $\text{vec}(a_3)$ is a $D$-dimensional vector whose elements are all equal to $a_3$. This transformation will shift and re-scale the original function, creating heterogeneous functional forms. The homogeneous scenario can be viewed as a special case when $a_1=1,a_2=a_3=0$. We set $K=10$. Other settings remain unchanged. Results are shown in Table \ref{Table:2}.

\begin{table}[!htbp]
\centering
\begin{tabular}{cccc}
\hline
\textbf{Functions}  & \texttt{CBOC-L} & \texttt{Individual} & \texttt{FedBO}   \\ \hline
\textbf{Levy-$2$} & \bm{$0.990(\pm 0.001)$} & $0.942(\pm 0.005)$ & $0.958(\pm 0.003)$ \\ \hline
\textbf{Levy-$4$} & \bm{$0.984(\pm 0.002)$} & $0.933(\pm 0.012)$  & $0.940(\pm 0.009)$ \\ \hline
\textbf{Levy-$8$} & \bm{$0.949(\pm 0.008)$} & $0.917(\pm 0.008)$  & 0.$903(\pm 0.011)$   \\ \hline
\end{tabular}
\caption{The average Gap across $K=10$ clients over 30 independent runs under a heterogeneous setting.}
\label{Table:2}
\end{table}

From Tables \ref{Table:1}-\ref{Table:2}, we can derive two key insights. First, collaborative methods yield superior performance than non-collaborative competitors through higher average Gap metrics. This evidences the importance of collaboration in improving the optimal design process. Second, \texttt{CBOC} outperforms all benchmarks. This credits to \texttt{CBOC}'s ability to address heterogeneity through a flexible consensus framework that allows clients to collaboratively explore and exploit the design space and eventually obtain client-specific solutions. 

% Second, \texttt{CBOC} is robust to the heterogeneous setup. This credits to the dynamic adjustment of the consensus matrix $\bm{W}^{(t)}$. In the late stages, the individual client will focus more on its design problem to get the optimal design points. 

\subsection{An illustrative example}
To visualize the performance of our method, we provide an illustration using the Levy-$2$ function, $K=2$, and the same heterogeneity structure as in Sec. \ref{subsec:test_1}. Specifically, the two design functions are set to $f_1(\bm{x})=f(\bm{x}+\text{vec}(1))+1, f_2(\bm{x})=2f(\bm{x}+\text{vec}(2))+2$. Here, each client starts with five two-dimensional designs, and then both \texttt{CBOC-L} and \texttt{Individual} are run for $T=40$. Fig. \ref{fig:path} demonstrates the landscape of the original Levy-$2$ function and shows contour plots of $f_1,f_2$. The selected design points for each client are marked as red for \texttt{CBOC-L} and green for \texttt{Individual}, and the iteration indices are labeled next to those points. We did not label all points for better visualization.

From Fig. \ref{fig:path}, we can see that through collaboration \texttt{CBOC-L} samples more frequently near optimal regions, as evident by the larger number of red points close to the optimal. This allows \texttt{CBOC-L} to reach the region that contains the optimal design (white-most region) for both clients in 25 iterations. On the other hand, \texttt{Individual} required around 38 iterations for client 1, while for client 2, even at $T=40$, the optimal region was still not explored yet. This again highlights the benefit of collaboration.  \textit{More importantly, Fig. \ref{fig:path} highlights the ability of our method to operate under heterogeneity as both clients were able to reach their distinct optimal design neighborhood, and they do so much faster than operating in isolation}.   

% It can be seen that clients collaborate to reach a neighbor of regions A and B that contains optimal solutions for both devices ($(0,0)$ for client 1 and $(-1,-1)$ for client 2). Afterward, each client will focus more on its individual design problem and obtain client-specific design points (also recall the illustrative example in Figure \ref{fig:change}). This credits to the specification of $\bm{W}^{(t)}$ that gradually shrinks off-diagonal elements to zero to dilute the impact of other clients.  

\begin{figure}[!htbp]
    \centering
    \centerline{\includegraphics[width=0.9\columnwidth]{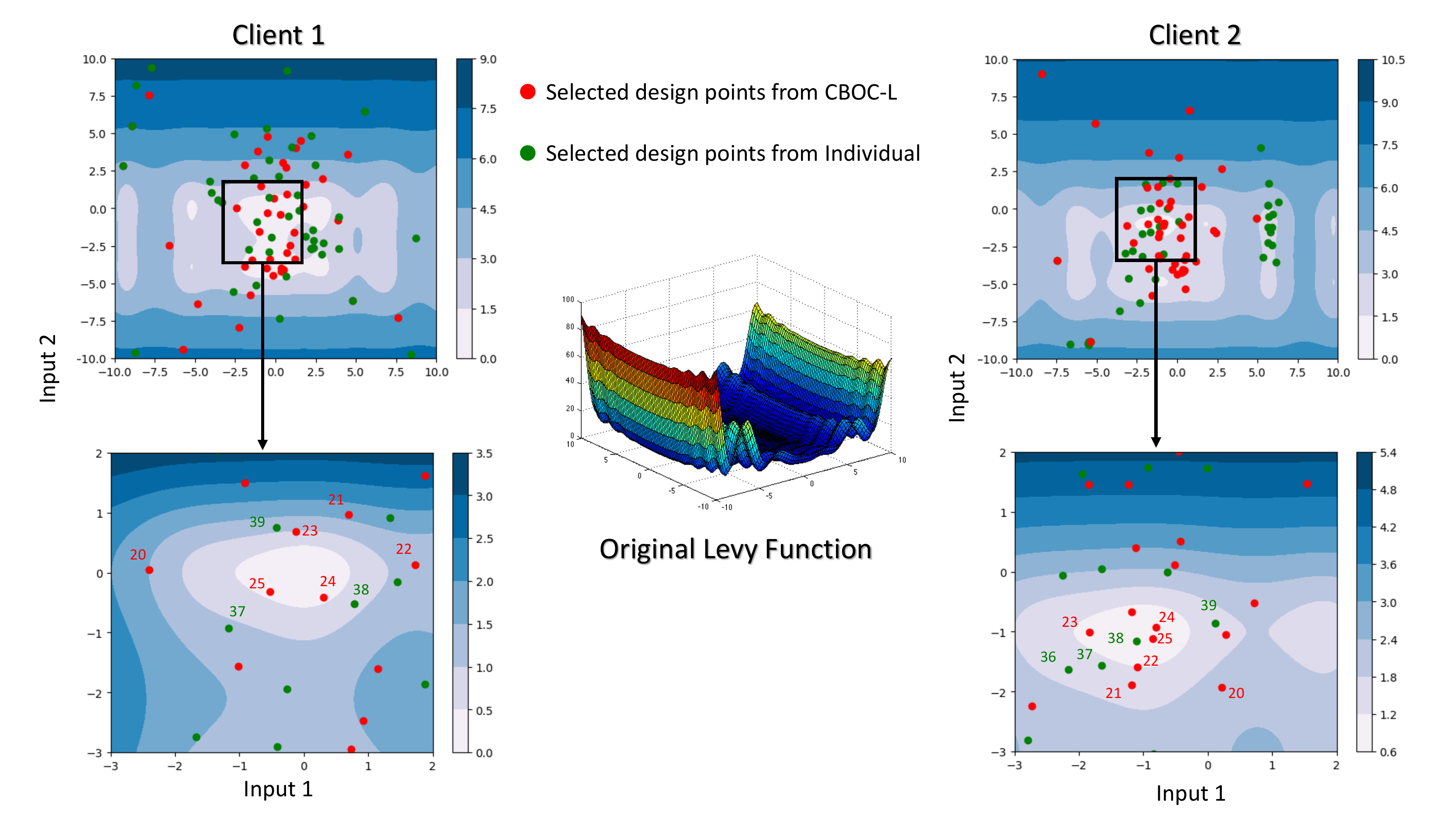}}
    \caption{Contour plot of $f_1,f_2$ and selected design points. The x-axis shows the first design input (i.e., $\tensor[^1]{x}{}$), and the y-axis shows the second design input (i.e., $\tensor[^2]{x}{}$). }
    \label{fig:path}
\end{figure}

\subsection{Testing Function 2: Shekel-$10$} 

Shekel-$10$ is a four-dimensional ($D=4$) with $10$ local minima. It has the functional form:
\begin{align*}
    f(\bm{x})=-\sum_{i=1}^{10}\left(\sum_{d=1}^4(\tensor[^d]{x}{}-F_{di})^2+\xi_i\right)^{-1}.
\end{align*}
We defer the specification of $F_{di}$ and $\xi_i$ to Appendix 3.

Similar to the heterogeneous scenario in Sec. \ref{subsec:test_1}, for each client $k$, we transform the Shekel-$10$ function $f(\bm{x})$ to $f_k(\bm{x})=a_1f(\bm{x}+\text{vec}(a_3))+a_2$, where $a_1\sim\texttt{Uniform}(0.5, 1)$ and $a_2\sim\mathcal{N}(0,2), a_3\sim\mathcal{N}(0,1)$. We test all benchmarks using $K=5,10,15$ and $20$ clients.

\begin{table}[!htbp]
\centering
\begin{tabular}{ccccc}
\hline
\textbf{Functions}  & \texttt{CBOC-L} & \texttt{CBOC-U} & \texttt{Individual} & \texttt{FedBO}   \\ \hline
$K=5$ & \bm{$0.475(\pm 0.053)$} & $0.462(\pm 0.052)$ & $0.350(\pm 0.055)$ & $0.370(\pm 0.047)$ \\ \hline
$K=10$ & \bm{$0.516(\pm 0.049)$} & $0.501(\pm 0.029)$ & $0.364(\pm 0.035)$ & $0.422(\pm 0.040)$    \\ \hline
$K=15$ & \bm{$0.577(\pm 0.033)$} & $0.553(\pm 0.036)$ & $0.356(\pm 0.022)$  & $0.496(\pm 0.053)$   \\ \hline
$K=20$ & \bm{$0.592(\pm 0.036)$} & $0.572(\pm 0.015)$ & $0.335(\pm 0.028)$ & $0.535(\pm 0.051)$   \\ \hline
\end{tabular}
\caption{The average Gap across $K$ clients over 30 independent runs under a heterogeneous setting.}
\label{Table:3}
\end{table}

Table \ref{Table:3} shows that \texttt{CBOC} outperforms both the non-collaborative method and state-of-the-art \texttt{FedBO}. Interestingly, we observe that the performance of \texttt{CBOC} improves as more clients participate in the collaboration process. For example, the average Gap for \texttt{CBOC} when $K=5$ is 0.475. As we increase $K$ to $20$, the average Gap becomes 0.592. Fig. \ref{fig:more_clients} plots the evolution of the average Gap with respect to iterations. This result further demonstrates the benefits of collaboration.

\begin{figure}[!htbp]
    \centering
    \centerline{\includegraphics[width=0.5\columnwidth]{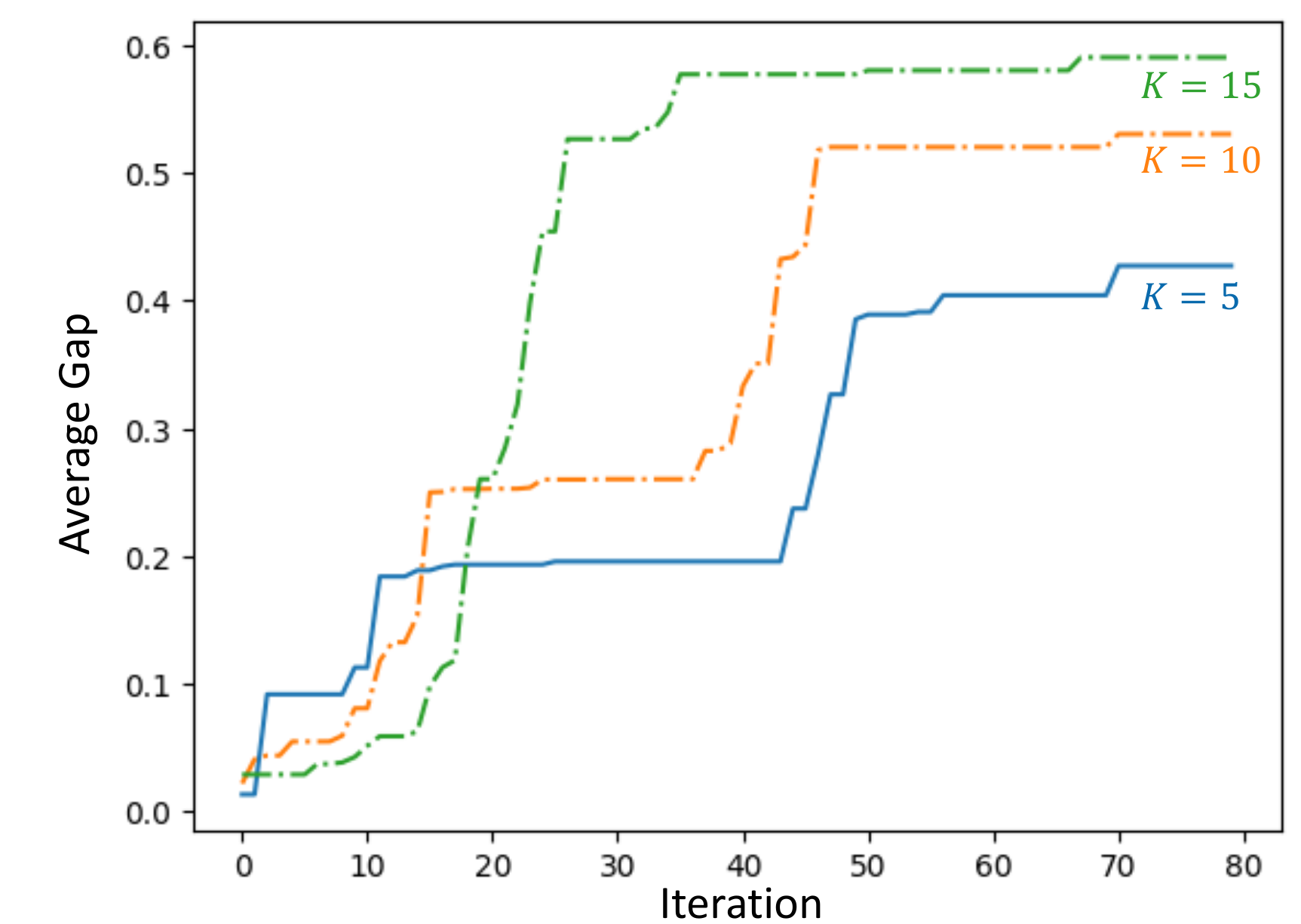}}
    \caption{The evolution of the average Gap with respect to iterations and collaborators.}
    \label{fig:more_clients}
\end{figure}

\subsection{Other Testing Functions}
Finally, we test our methods on three other functions:  Branin, Ackley-$D$, and Harmann-$6$. We use $K=10$. Experimental details and function specifications are deferred to the appendix.

\begin{table}[!htbp]
\centering
\begin{tabular}{ccccc}
\hline
\textbf{Functions}  & \texttt{CBOC-L} & \texttt{CBOC-U} & \texttt{Individual} & \texttt{FedBO}   \\ \hline
Branin & \bm{$0.992(\pm 0.000)$} & $0.990(\pm 0.001)$ & $0.975(\pm 0.001)$ & $0.986(\pm 0.001)$ \\ \hline
Ackley-$5$ & $\bm{0.656(\pm 0.057)}$ & $0.641(\pm 0.053)$ & $0.501(\pm 0.039)$ & $0.632(\pm 0.041)$    \\ \hline
Hartmann-$6$ & \bm{$0.968(\pm 0.005)$} & $0.959(\pm 0.002)$ & $0.941(\pm 0.001)$ & $0.955(\pm 0.003)$  \\ \hline
% Levy-$20$ & & &     \\ \hline
\end{tabular}
\caption{The average Gap across $K=10$ clients over 30 independent runs under a heterogeneous setting.}
\label{Table:4}
\end{table}

Similar to our previous analysis, Table \ref{Table:4} shows the superiority of our approach.   

\section{Case Study}
\label{sec:real}

A case study on consensus BO-driven closed-loop biosensor design optimization was performed (see Fig. \ref{fig:CPScurrent}). The closed-loop workflow was based on simulation (FEA) guided by a decentralized computing process for collaboration (i.e., consensus BO). Given the demand for device-based biosensors for industrial process analytical technology (e.g., for bioprocess monitoring and control) and health monitoring (e.g., via wearable sensors), the case study was focused on closed-loop optimization of a device-based biosensor. Specifically, our case study focused on optimizing a device-based biosensor with a milli-scale transducer used in a continuous flow format, which is consistent with several types of device-based biosensors, including electrochemical, mechanical, and electromechanical biosensors.  

\begin{figure}[htbp!]
\vspace{-1em}
	\centering
	\makebox[\textwidth]{\includegraphics[keepaspectratio=true,width=1\textwidth]{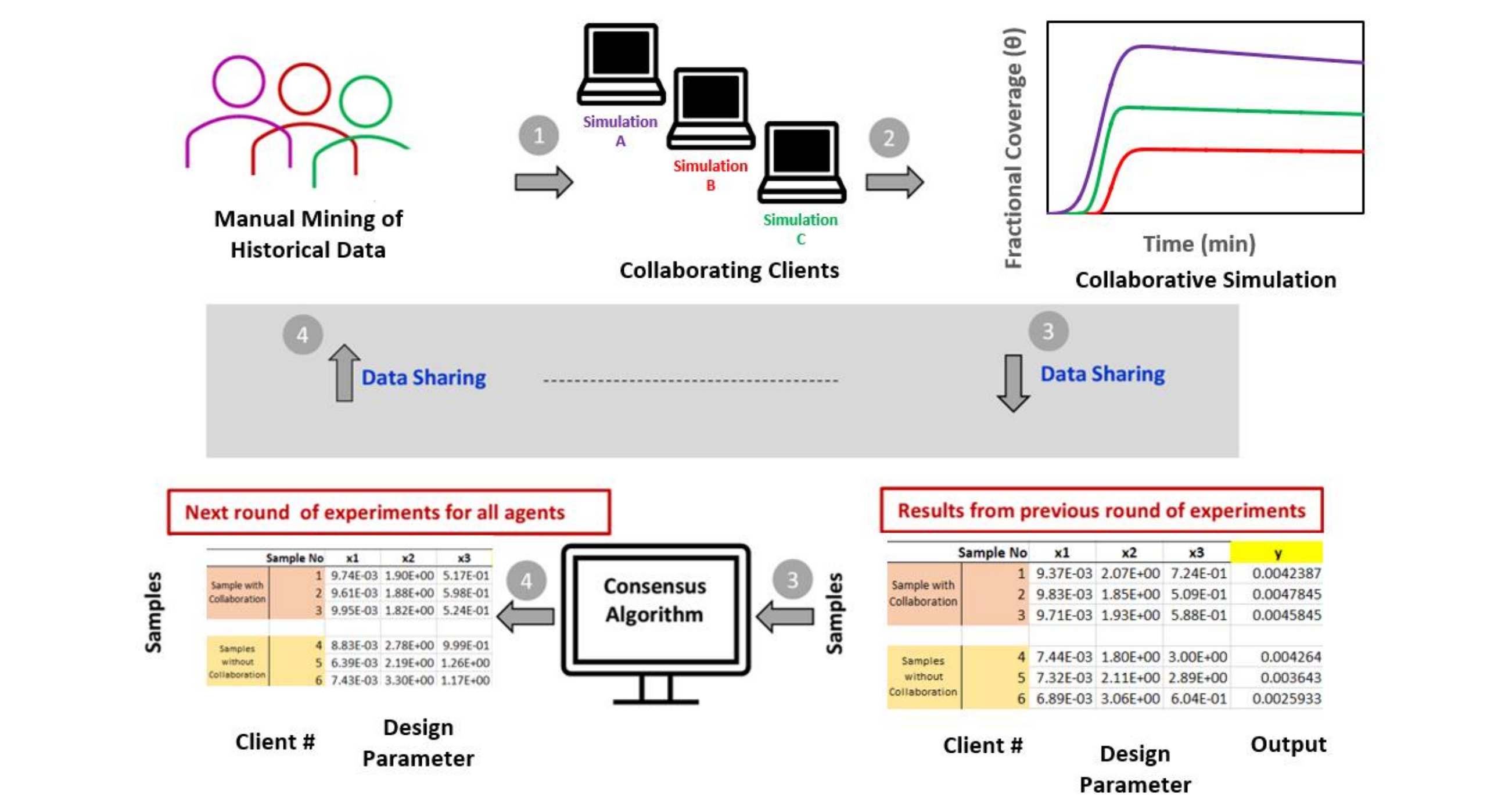}}
	\caption{\label{fig:CPScurrent} Case study schematic.}
	\vspace{-1em}
\end{figure}

Biosensor performance, specifically, transient binding of the target analyte, was calculated using commercially available FEA software (COMSOL Multiphysics, COMSOL). The transient fractional surface coverage of the bound target analyte on the biosensor surface was calculated by numerical solution of a coupled convection-diffusion-reaction model in 2D using a time-dependent study. The convection-diffusion-reaction model was constructed by coupling laminar flow and transport of diluted species physics with a surface reaction. The computational domain consisted of a 6 $\times$ 1 mm$^2$ (length × width) rectangular fluidic channel that encompassed the domain $x \in [-3, 3]$ mm and $y \in [0, 1]$ mm. The biosensor surface (i.e., surface 1) on which the binding reaction between immobilized biorecognition elements and target analyte occurred encompassed the domain $x \in [0, 1]$ mm at $y = 0$. The computational domain contained five additional surfaces on which boundary and initial conditions were applied. The coupled partial differential equations associated with the convection-diffusion-reaction model were solved subject to the following boundary and initial conditions:    

For laminar flow physics:

\begin{itemize}
    \item Surface 2 (inlet): normal inflow velocity = $u_{\text{in}}$.
    \item Surface 3 (outlet): static pressure = 0 Pa.
    \item Surfaces 1 and 4-6 (top and bottom walls, including the sensing surface): no slip.
\end{itemize}
Here, no slip represents a traditional no slip boundary condition. The fluid properties were defined by water and obtained from the FEA software’s material library. The initial velocity and pressure fields were zero. 

For transport of diluted species physics: 

\begin{itemize}
    \item Surface 1: $J_{0,c}=-r_{\text{ads}} + r_{\text{des}}$.
    \item Surface 2: $c=c_0$.
    \item Surface 3: $J_{0,c}=\bm{n}\cdot D_i\nabla c_i$.
    \item Suraces 4-6: no flux.
\end{itemize}
In the above equations, $J_{0,c}$ is the mass flux, $r_{\text{ads}}=k_{\text{on}}cc_s$ is the adsorption (binding) rate of the target analyte, $c$ is the concentration of the target analyte in solution, $c_s$ is the concentration of immobilized biorecognition element, $r_{\text{des}}=k_{\text{off}}c_s^*$ is the desorption rate, $c_s^*$ is the concentration of occupied sites, $c_0 = c_{\text{in}}\text{GP}(t)$, $c_{\text{in}} = 1$ nM is the concentration of the target analyte in the injected sample, $\text{GP}(t)$ is a time-dependent Gaussian pulse function with integral normalization and integral value of unity, and no flux represents a traditional no flux boundary condition. The following values of transport properties and rate constants were used in the simulation: $D_i = 1 \times 10^{-11}$ m$^2$/s is the diffusivity of the target analyte in solution, the binding (forward) rate constant ($k_{\text{on}}) = 1 \times 10^6$ M$^{-1}$s$^{-1}$, and the unbinding (reverse) rate constant ($k_{\text{off}}$) $= 1 \times 10^{-3}$ s$^{-1}$. The initial concentration field was zero.

For surface reaction chemistry: 

\begin{itemize}
    \item Surface 1: reaction rate $=r_{\text{ads}} - r_{\text{des}}$.
    \item Surfaces 2-6: no reaction.
\end{itemize}
The following surface properties were used: $\rho_s =$ density of surface sites and site occupancy number $= 1$. The surface diffusion of the bound target analyte was assumed to be zero. The initial concentration of bound target was zero.

The convection-diffusion-reaction model was discretized and solved subject to the aforementioned boundary and initial conditions using a physics-controlled adaptive mesh that contained 373004 (plus 351428 degrees of freedom) in the final mesh for a given combination of inputs $[x_1, x_2, x_3] = [u_{\text{in}}, \rho_s, \text{GP}_{\text{std}}]$, where $\text{GP}_{\text{std}}$ is the standard deviation of the Gaussian pulse associated with the injected sample. The specific values of $u_{\text{in}}, \rho_s, \text{GP}_{\text{std}}$ examined were systematically selected by traditional BO and the consensus BO model, which are referred to as ``non-collaborative” and ``collaborative” learning, respectively.

The simulation’s output of interest (i.e., the quantity to be optimized) was the maximum fractional surface coverage of the bound target analyte ($\theta_{\text{max}}$) during the transient binding response, which is a fundamental characteristic of biosensor function that is associated with the performance characteristics of sensitivity, detection limit, dynamic range, and speed.

The objective of this case study was to discover (i.e., learn) the biosensor design and measurement format parameters that maximize $\theta_{\text{max}}$ (i.e., the maximum amount of captured target analyte). An optimal solution (i.e., design) was sought within the design space of $x_1 \in [1 \times 10^{-4}, 1 \times 10^{-2}]$ m/s, $x_2 \in [1.8 \times 10^{-8}, 3.3 \times 10^{-8}]$ mol/m$^2$, and $x_3 \in [0.5, 3]$ min.  The range and limits of the design space were selected based on practical values used in previous biosensing studies \citep{selmi2017optimization, squires2008making, baronas2021biosensors, johnson2012sample}.  In summary, the objective of this case study was to learn the optimal combination of $u_{\text{in}}, \rho_s, \text{GP}_{\text{std}}$ that maximized the amount of captured target analyte and to compare the learning performance achieved by traditional BO performed by a group of independent clients that do not share information with that achieved by collaborative learning performed by a group of clients that share information and make decisions for next-to-test values using consensus BO.

The case study involved two groups: a non-collaborative group and a collaborative group, each consisting of three clients. The non-collaborative group used conventional centralized BO algorithms, where each client solves its own problem and there is no communication among clients. In contrast, the collaborative group utilized the \texttt{CBOC} algorithm (Algorithm \ref{alg::consensus}) to collaboratively find optimal solutions. For \texttt{CBOC}, we adopted the leader-driven matrix mentioned in Sec. \ref{subsec:matrix}. Fig. \ref{fig:trend} shows the trend of biosensor performance, specifically maximum amount of captured target species, throughout the iterative closed-loop workflow. As shown in Fig. \ref{fig:trend}, the self-driving workflow driven by consensus BO was capable of learning an optimized biosensor design and measurement format parameter selection after ten rounds of experimentation as evidenced by the trend of $\theta_{\text{max}}$. Additionally, the average biosensor performance was higher and exhibited lower variance in the collaborative group relative to the non-collaborative group, which highlights the value of leveraging collaboration in closed-loop high-throughput experimentation.

\begin{figure}[!htbp]
    \centering
    \centerline{\includegraphics[width=0.8\columnwidth]{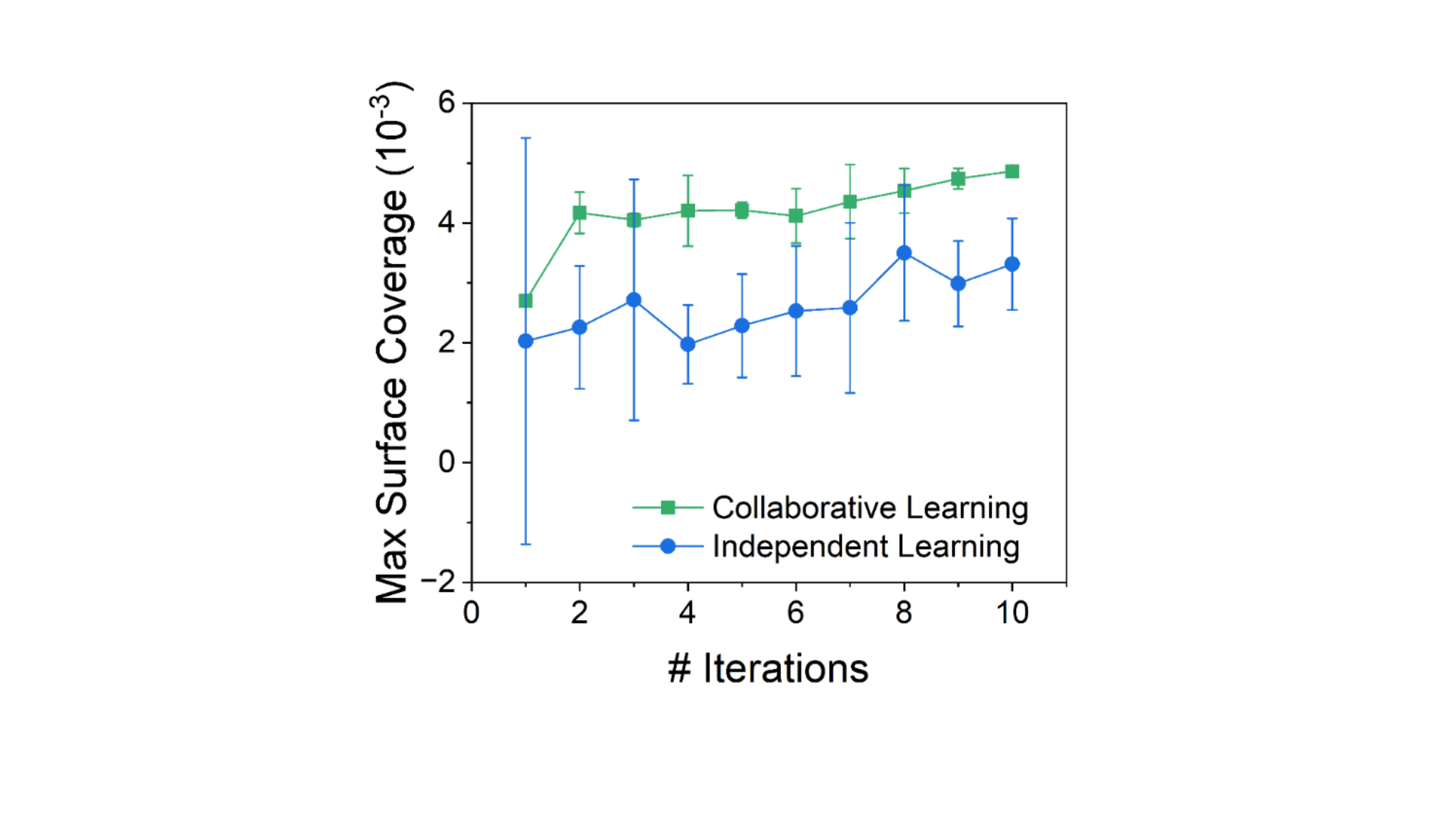}}
    \caption{Trend in the biosensor’s maximum fractional surface coverage (i.e., the maximum amount of captured target analyte) vs. iteration number (i.e., sequential rounds of simulation) via BO without and with collaboration after ten iterations. }
    \label{fig:trend}
\end{figure}

This case study serves as a proof-of-concept showcasing that collaboration in optimal design can reap benefits for the participating entities.

\section{Discussion and Conclusion}
\label{sec:con}

In this paper, we propose a collaborative Bayesian optimization framework built upon a consensus mechanism. Our experiments on simulated data and a real-world sensor design case study show that collaboration through our proposed framework can accelerate and improve the optimal design process.

Collaborative BO is still in its infancy stage, and few prior studies exist along this line. Indeed, there are many avenues of potential improvement. These avenues include: (i) extending the consensus framework to multi-objective and multi-fidelity settings, (ii) developing a resource-aware approach that allows some clients to perform more or fewer experiments depending on their resources, (iii) understanding the theoretical conditions needed for the collaborative design to outperform its non-collaborative counterpart. If successful, this may lead to targeted algorithms that exploit these conditions to improve performance. We hope this research may inspire future work along these avenues and beyond in the future.

\newpage

\appendix

\section*{Appendix}

\section{Technical Results}

Denote by $\bm{x}_k^*=\argmax_{\bm{x}} f_k(\bm{x})$ the global optimal solution of the design function for client $k$. Recall that, in Algorithm 1, $\bm{x}^{(t)\text{new}}_k=[(\bm{W}^{(t)}\otimes\bm{I}_D)\bm{x}_{\mathcal{C}}^{(t)}]_k$. Let $r^{(t)}_{k}=f_{k}(\bm{x}_k^*)-f_{k}(\bm{x}^{(t)\text{new}}_k)$ be the regret for client $k$ at iteration $t$. \textbf{Our goal is to derive an upper bound on the cumulative regret $R_{k,T}=\sum_{t=1}^Tr^{(t)}_{k}$ and show that $R_{k,T}$ has a sublinear growth rate.} We first define a stopping criterion for the expected improvement utility function $\texttt{EI}_k^{(t)}$. By definition, $\texttt{EI}_k^{(t)}(\bm{x})\coloneqq\mathbb{E}_{\hat{f}_k|\mathcal{D}^{(t)}_k}\left[U(\hat{f}_k(\bm{x}))\right]=\mathbb{E}_{\hat{f}_k|\mathcal{D}^{(t)}_k}\left[(\hat{f}_k(\bm{x})-y_k^{*(t)})^+\right]\geq 0, \forall \bm{x}\in\mathbb{R}^D$. Therefore, we define a small positive constant $\kappa$ such that when $\texttt{EI}_k^{(t)}(\bm{x})<\kappa$, we stop our algorithm at iteration $t$. Next, we will present the assumptions we made for our main Theorem.

\begin{assumption}
\label{assumption:1}
(Homogeneity) Design functions are homogeneous (i.e., $f_1(\bm{x})=f_2(\bm{x})=\cdots=f_K(\bm{x})$ for all $\bm{x}\in\mathbb{R}^D$).
\end{assumption}

% Assumption \ref{assumption:1} states that all clients' design functions are homogeneous.

\begin{assumption}
\label{assumption:4}
(Initial Sample Size) Without loss of generality, we assume the initial sample size is 2 such that $N^{(t)}_k=2+t$, for all $k\in[K]$, at iteration $t$.
\end{assumption}

\begin{assumption}
\label{assumption:2}
(Bounded Utility Solutions) Assume $\norm{\bm{x}^{(t)}_k-\bm{x}^{(t)}_{j}}_2\leq \mathcal{O}(\frac{1}{(\log (2+t))^{0.5+\epsilon}})$ for all $k,j\in[K]$, where $\epsilon$ is a positive constant.
\end{assumption}

\begin{assumption}
\label{assumption:3}
(Bounded Variance) Assume $|\sigma_k(\bm{x}^{(t)}_k)-\sigma_k(\bm{x}^{(t)\textup{new}}_k)|\leq \mathcal{O}(\frac{1}{(\log (2+t))^{0.5+\epsilon}})$, for all $k\in[K]$, where $\epsilon$ is a positive constant. 
\end{assumption}

% Without loss of generality, we assume the initial sample size is 2 such that $N^{(t)}_k=2+t$ at iteration $t$.

Assumption \ref{assumption:2} states that the distance between design points returned by any two clients is bounded by $\mathcal{O}(\frac{1}{(\log (2+t))^{0.5+\epsilon}})$ at every iteration $t$. Recall that $\bm{x}^{(t)}_k=\argmax_{\bm{x}}\texttt{EI}^{(t)}_k(\bm{x})$. Furthermore, Assumption \ref{assumption:3} assumes that the difference between the $\mathcal{GP}$ posterior variance evaluated at $\bm{x}^{(t)}_k$ and that evaluated at the consensus solution is bounded. These two assumptions are intuitively understandable: at the early stages, since clients have very little data, their surrogate functions will differ even under a homogeneous setting. Therefore, the error term $\mathcal{O}(\frac{1}{(\log (2+t))^{0.5+\epsilon}})$ is large (i.e., due to small $t$). As the collaborative process proceeds, the error term decreases, and gradually the solutions returned by all clients will be close since the underlying true design functions are the same. Eventually if  $\bm{x}_1^{(t)}=\bm{x}_2^{(t)}=\cdots=\bm{x}_K^{(t)}$, then $\bm{W}^{(t)}\bm{x}^{(t)}_{\mathcal{C}}=\bm{x}^{(t)}_{\mathcal{C}}$, since $\bm{x}^{(t)}_{\mathcal{C}}$ concatenates all $\{\bm{x}_k^{(t)}\}_{k=1}^K$. The variance difference $|\sigma_k(\bm{x}^{(t)}_k)-\sigma_k(\bm{x}^{(t)\text{new}}_k)|$ will be zero accordingly.

One natural question is: why do we believe the error term is decreasing at a rate of $\mathcal{O}\left(\frac{1}{(\log N^{(t)}_k)^{0.5+\epsilon}}\right) $? Fortunately, this argument can be justified using the following Corollary.
\begin{corollary}
\label{cor:1}
\citep{lederer2019uniform} Suppose $f_k(\cdot)$ is sampled from a zero mean $\mathcal{GP}$ defined through a Lipschitz continuous kernel function $\mathcal{K}_k(\cdot,\cdot)$. Given $N_k$ collected observations, if there exists an $\epsilon$ such that $\sigma_k(\bm{x})\in\mathcal{O}(\frac{1}{(\log N_k)^{0.5+\epsilon}})$, $\forall \bm{x}\in\mathbb{R}^D$, then it holds for every $\delta_0\in(0,1)$ that
\begin{align*}
    \mathbb{P}\left(\sup_{\bm{x}}\norm{\mu_k(\bm{x})-f_{k}(\bm{x})}\in\mathcal{O}\left(\frac{1}{(\log N_k)^\epsilon}-\log(\delta_0)\frac{1}{(\log N_k)^{0.5+\epsilon}}\right) \right)\geq 1-\delta_0.
\end{align*}
\end{corollary}
This Lemma implies that the $\mathcal{GP}$ surrogate has an approximation error that is decreasing at a rate of $\mathcal{O}\left(\frac{1}{(\log N^{(t)}_k)^{0.5+\epsilon}}\right) $. This error bound decreases as more samples are collected (larger $N^{(t)}_k$). 

%Recall that, in the BO, the sample size $N^{(t)}_k$ and iteration index $t$ are highly related since each client collects one sample at each iteration. 

% Without loss of generality, assume the initial sample size for each client is $2$, then we can replace $N^{(t)}_k$ by $t+2$, at iteration $t$.

%Now, as more samples are collected, the client k's surrogate will become more accurate and closer to $f_k(\cdot)$. Therefore, the next design points $\{\bm{x}^{(t)}_k\}_{k=1}^K$ decided by all clients 

% Here note that in the context of sequential design, the sample size $N$ is proportional to the iteration $t$. 

% \begin{assumption}
% \label{assumption:1}
%     For each client $k$, the unknown function $f_k(\cdot)$ is sampled from a $\mathcal{GP}(0, \mathcal{K}(\cdot,\cdot))$ and the output $y_{k,i}=f_k(\bm{x}_{k,i})+\epsilon_{k,i}$, where $\epsilon_{k,i}\overset{i.i.d.}{\sim}\mathcal{N}(0,\sigma_k^2)$. 
% \end{assumption}

To this end, we state our main Theorem and its proof below. All supporting lemmas are in Sec. \ref{sec:2}.

% \begin{theorem}
% \label{theorem:1}
% Let $\kappa>0$ be a constant used for the stopping criterion. Suppose a squared exponential kernel is used for the $\mathcal{GP}$ surrogate and Assumption \ref{assumption:1} holds, then with probability at least $1-\delta$, the cumulative regret after $T$ iterations has a sublinear growth rate: $R_{k,T}=\sum_{t=1}^Tr^{(t)}_{k}\sim\mathcal{O}(\sqrt{T\times(\log T)^{D+4}})$. Moreover, $\lim_{T\to\infty}\frac{R_{k,T}}{T}=0$.
% \end{theorem}

\begin{theorem-non}
\label{theorem:1_app}
Suppose a squared exponential kernel function $\mathcal{K}_k(\bm{x}_1,\bm{x}_2)=u_{k}^2\exp\left(\frac{\norm{\bm{x}_1-\bm{x}_2}^2}{2\ell_k^2}\right)$ is used for the $\mathcal{GP}$ surrogate and each client uses the {\normalfont\texttt{EI}} utility. Without loss of generality, assume $\mathcal{K}_k(\bm{x}_1,\bm{x}_2)\leq 1$. Suppose assumptions \ref{assumption:1}-\ref{assumption:3} hold, and the algorithm will stop when $\texttt{EI}_k^{(t)}<\kappa$, then given any doubly stochastic $\bm{W}^{(t)}$ with non-negative elements, with probability at least $(1-\frac{\delta_1}{T})^T$, where $\delta_1\in(0,T)$, the cumulative regret after $T>1$ iterations is 
\begin{align*}
    R_{k,T}&=\sum_{t=1}^Tr_{k}^{(t)}\leq \sqrt{\frac{6T\left[(\log T)^3+1+C \right](\log T)^{D+1}}{\log(1+v_k^{-2})}}+\sqrt{\frac{2T(\log T)^{D+4}}{\log(1+v_k^{-2})}}+\sum_{t=1}^T\mathcal{O}\left(\frac{1}{(\log (2+t))^{0.5+\epsilon}}\right)\\
    &\sim\mathcal{O}(\sqrt{T\times(\log T)^{D+4}}),
\end{align*}
where $C=\log[\frac{1}{2\pi\kappa^2}]$.
\end{theorem-non}

\begin{proof}
Recall that we define the regret as $r_{k}^{(t)}=f_{k}(\bm{x}_k^*)-f_{k}(\bm{x}^{(t)\text{new}}_k)$. We decompose the regret $r_{k}^{(t)}$ as
% \begin{align*}
%     R_t=\sum_{k=1}^K R_{k,t}=\sum_{k=1}^K\bigg\{f_{k,\text{True}}(\bm{x}_k^*)- f_{k,\text{True}}([\bm{x}^{(t)new}]_k)\bigg\}.
% \end{align*}
\begin{align*}
    r_{k}^{(t)}=\bigg\{\underbrace{f_{k}(\bm{x}_k^*)- y^{*(t)}_{k} }_{\text{A}} + \underbrace{y^{*(t)}_{k}- f_{k}(\bm{x}^{(t)\text{new}}_k)}_{\text{B}}\bigg\},
\end{align*}
where $y^{*(t)}_{k}=\max \bm{y}^{(t)}_k$ is the largest observed output value from iteration $t$. 

By Lemma \ref{lemma:2}, we have, with probability $1-\delta$ ($\delta\in(0,1)$), ${\normalfont\texttt{EI}}^{(t)}_k(\bm{x})\geq I_k^{(t)}(\bm{x})-\sqrt{\beta_k^{(t)}}\sigma_k^{(t)}(\bm{x})$, where $I_k^{(t)}(\bm{x})=(f_{k}(\bm{x}) -y^{*(t)}_{k})^+$ and $\beta^{(t)}_{k}$ is defined in Lemma \ref{lemma:1}. Therefore, we obtain
\begin{align*}
    &{\normalfont\texttt{EI}}^{(t)}_k(\bm{x}_k^*)\geq I_k^{(t)}(\bm{x}_k^*)-\sqrt{\beta_k^{(t)}}\sigma_k^{(t)}(\bm{x}_k^*),\\
    &\Rightarrow I_k^{(t)}(\bm{x}_k^*)\leq{\normalfont\texttt{EI}}^{(t)}_k(\bm{x}_k^*)+\sqrt{\beta_k^{(t)}}\sigma_k^{(t)}(\bm{x}_k^*).
\end{align*}
If $f_{k}(\bm{x})>y^{*(t)}_{k}$, we have $I_k^{(t)}(\bm{x}_k^*)=f_{k}(\bm{x}_k^*)-y^{*(t)}_{k}$. Otherwise, we have $I_k^{(t)}(\bm{x}_k^*)=0\geq f_{k}(\bm{x}_k^*)-y^{*(t)}_{k}$. Under both situations, we bound term A as
\begin{align*}
    \text{A}&\leq \bigg\{ \texttt{EI}^{(t)}_k(\bm{x}_k^*)+\sqrt{\beta^{(t)}_{k}}\sigma^{(t)}_{k}(\bm{x}_k^*)\bigg\},
    %\leq \sum_{k=1}^K \bigg\{{\color{red} \texttt{EI}_k([\bm{x}^{new}]_k)+\sqrt{\beta_{k,c}}\sigma_{c-1}(\bm{x}_k^*)}\bigg\}\\
    %&\leq \sum_{k=1}^K \bigg\{{\color{blue}\sigma_{c-1}([\bm{x}^{new}]_k)\tau(z_{c-1}([\bm{x}^{new}]_k))}+\sqrt{\beta_{k,c}}\sigma_{c-1}(\bm{x}_k^*)\bigg\},
\end{align*}
with probability $1-\delta$. Next, we want to study the relationship between $\texttt{EI}^{(t)}_k(\bm{x}_k^*)$ and $\texttt{EI}_k^{(t)}(\bm{x}^{(t)}_k)$ to further bound term A.

\paragraph{Relationship between $\texttt{EI}^{(t)}_k(\bm{x}_k^*)$ and $\texttt{EI}_k^{(t)}(\bm{x}^{(t)}_k)$} By definition and Lemma \ref{lemma:2}, we have
\begin{align*}
    \texttt{EI}_k^{(t)}(\bm{x}^{(t)}_k)&=\sigma_k^{(t)}(\bm{x}^{(t)}_k)\phi(z_k^{(t)}(\bm{x}^{(t)}_k))+(\mu^{(t)}_k(\bm{x}^{(t)}_k)-y^{*(t)}_{k})\Phi(z_k^{(t)}(\bm{x}^{(t)}_k))\\
    &=\sigma_k^{(t)}(\bm{x}^{(t)}_k)\tau(z_k^{(t)}(\bm{x}^{(t)}_k)),
\end{align*}
where $z_k^{(t)}(\cdot)$ and $\tau(\cdot)$ have been defined in Lemma \ref{lemma:2}. Therefore,
\begin{align*}
    &|\texttt{EI}^{(t)}_k(\bm{x}_k^{(t)})-\texttt{EI}^{(t)}_k(\bm{x}^{(t)\text{new}}_k)|\\
    &=|\sigma_k^{(t)}(\bm{x}^{(t)}_k)\tau(z_k^{(t)}(\bm{x}^{(t)}_k))-\sigma_k^{(t)}(\bm{x}^{(t)\text{new}}_k)\tau(z_k^{(t)}(\bm{x}^{(t)\text{new}}_k))|.
\end{align*}
By Assumption \ref{assumption:3}, we have $|\sigma_k(\bm{x}^{(t)}_k)-\sigma_k(\bm{x}^{(t)\textup{new}}_k)|\leq \mathcal{O}(\frac{1}{(\log (2+t))^{0.5+\epsilon}})$. This implies 
\begin{align*}
    \sigma_k(\bm{x}^{(t)}_k) -\mathcal{O}(\frac{1}{(\log (2+t))^{0.5+\epsilon}})\leq \sigma_k(\bm{x}^{(t)\textup{new}}_k)\leq \sigma_k(\bm{x}^{(t)}_k) +\mathcal{O}(\frac{1}{(\log (2+t)^{0.5+\epsilon}}).
\end{align*}
Since $\texttt{EI}^{(t)}_k(\bm{x}_k^{(t)})\geq \texttt{EI}^{(t)}_k(\bm{x}^{(t)\text{new}}_k)$ ($\bm{x}_k^{(t)}$ is the maximizer of $\texttt{EI}_k$), we take the lower bound of $\sigma_k(\bm{x}^{(t)\textup{new}}_k)$ and get
\begin{align*}
    &|\texttt{EI}^{(t)}_k(\bm{x}_k^{(t)})-\texttt{EI}^{(t)}_k(\bm{x}^{(t)\text{new}}_k)|\\
    &\leq |\sigma_k^{(t)}(\bm{x}^{(t)}_k)\tau(z_k^{(t)}(\bm{x}^{(t)}_k))-(\sigma_k^{(t)}(\bm{x}^{(t)}_k)-\mathcal{O}(\frac{1}{(\log (2+t))^{0.5+\epsilon}} ))\tau(z_k^{(t)}(\bm{x}^{(t)\text{new}}_k))|\\
    &= |\sigma_k^{(t)}(\bm{x}^{(t)}_k)(\tau(z_k^{(t)}(\bm{x}^{(t)}_k))-\tau(z_k^{(t)}(\bm{x}^{(t)\text{new}}_k)) ) + \mathcal{O}(\frac{1}{(\log (2+t))^{0.5+\epsilon}} )\tau(z_k^{(t)}(\bm{x}^{(t)\text{new}}_k))|.
\end{align*}
We know $\tau(\cdot)$ is Lipschitz continuous on a bounded domain. Assume the Lipschitz constant is $L_\tau$, we obtain
\begin{align}
\label{inequality:1}
    &|\texttt{EI}^{(t)}_k(\bm{x}_k^{(t)})-\texttt{EI}^{(t)}_k(\bm{x}^{(t)\text{new}}_k)|\nonumber\\
    &\leq |\sigma_k^{(t)}(\bm{x}^{(t)}_k)(\tau(z_k^{(t)}(\bm{x}^{(t)}_k))-\tau(z_k^{(t)}(\bm{x}^{(t)\text{new}}_k)) ) + \mathcal{O}(\frac{1}{(\log (2+t))^{0.5+\epsilon}} )\tau(z_k^{(t)}(\bm{x}^{(t)\text{new}}_k))|\nonumber\\
    &\leq L_\tau|\sigma_k^{(t)}(\bm{x}^{(t)}_k)|  |z_k^{(t)}(\bm{x}^{(t)}_k)-z_k^{(t)}(\bm{x}^{(t)\text{new}}_k)|+ |\mathcal{O}(\frac{1}{(\log (2+t))^{0.5+\epsilon}} )\tau(z_k^{(t)}(\bm{x}^{(t)\text{new}}_k))|
\end{align}
Recall that $z_k^{(t)}(\bm{x})=\frac{\mu^{(t)}_k(\bm{x})-y^{*(t)}_{k}}{\sigma_k^{(t)}(\bm{x})}$ for any $\bm{x}\in\mathbb{R}^D$, we then obtain
\begin{align*}
     &|z_k^{(t)}(\bm{x}^{(t)}_k)-z_k^{(t)}(\bm{x}^{(t)\text{new}}_k)|\\
     &=\left|\frac{\mu^{(t)}_k(\bm{x}^{(t)}_k)-y^{*(t)}_{k}}{\sigma_k^{(t)}(\bm{x}^{(t)}_k)}-\frac{\mu^{(t)}_k(\bm{x}^{(t)\text{new}}_k)-y^{*(t)}_{k}}{\sigma_k^{(t)}(\bm{x}^{(t)\text{new}}_k)}\right|\\
     &=\left|\frac{\sigma_k^{(t)}(\bm{x}^{(t)\text{new}}_k)\mu^{(t)}_k(\bm{x}^{(t)}_k)-\sigma_k^{(t)}(\bm{x}^{(t)\text{new}}_k)y^{*(t)}_{k}}{\sigma_k^{(t)}(\bm{x}^{(t)}_k)\sigma_k^{(t)}(\bm{x}^{(t)\text{new}}_k)}-\frac{\sigma_k^{(t)}(\bm{x}^{(t)}_k)\mu^{(t)}_k(\bm{x}^{(t)\text{new}}_k)-\sigma_k^{(t)}(\bm{x}^{(t)}_k)y^{*(t)}_{k}}{\sigma_k^{(t)}(\bm{x}^{(t)}_k)\sigma_k^{(t)}(\bm{x}^{(t)\text{new}}_k)}\right|\\
     &=\left|\frac{\sigma_k^{(t)}(\bm{x}^{(t)\text{new}}_k)\mu^{(t)}_k(\bm{x}^{(t)}_k)-\sigma_k^{(t)}(\bm{x}^{(t)\text{new}}_k)y^{*(t)}_{k}-\sigma_k^{(t)}(\bm{x}^{(t)}_k)\mu^{(t)}_k(\bm{x}^{(t)\text{new}}_k)+\sigma_k^{(t)}(\bm{x}^{(t)}_k)y^{*(t)}_{k}}{\sigma_k^{(t)}(\bm{x}^{(t)}_k)\sigma_k^{(t)}(\bm{x}^{(t)\text{new}}_k)}\right|.
\end{align*} 
Now, to bound $|z_k^{(t)}(\bm{x}^{(t)}_k)-z_k^{(t)}(\bm{x}^{(t)\text{new}}_k)|$, we need to study the properties of $\mu_k$ and $\sigma_k$.

\cite{lederer2019uniform} show that $\mathcal{GP}$ with a squared exponential kernel has a Lipschitz continuous posterior mean function. Assume the Lipschitz constant is $L_{\mu}$. Therefore,  we have $|\mu^{(t)}_k(\bm{x}_k^{(t)\text{new}})-\mu^{(t)}_k(\bm{x}_k^{(t)})|\leq L_{\mu}\norm{\bm{x}_k^{(t)\text{new}}-\bm{x}_k^{(t)}}_2$. By Assumption \ref{assumption:2}, we know $\norm{\bm{x}_k^{(t)}-\bm{x}_j^{(t)}}_2\leq\mathcal{O}(\frac{1}{(\log (2+t))^{0.5+\epsilon}})$. Therefore, we can show that
\begin{align*}
    \norm{\bm{x}_k^{(t)\text{new}}-\bm{x}_k^{(t)}}_2&=\norm{\sum_{j=1}^Kw^{(t)}_{kj}\bm{x}_j^{(t)}-\bm{x}_k^{(t)}}_2\\
    &=\norm{\sum_{j=1}^Kw^{(t)}_{kj}\bm{x}_j^{(t)}-\sum_{j=1}^Kw^{(t)}_{kj}\bm{x}_k^{(t)}}_2 & & \text{Since $\bm{W}^{(t)}$ is doubly stochastic}\\
    &\leq\mathcal{O}(\frac{1}{(\log (2+t))^{0.5+\epsilon}}) .
\end{align*}
Therefore, we know $|\mu^{(t)}_k(\bm{x}_k^{(t)\text{new}})-\mu^{(t)}_k(\bm{x}_k^{(t)})|$ is of order $\mathcal{O}(\frac{1}{(\log (2+t))^{0.5+\epsilon}}).$ Without loss of generality, we assume $\mu_k^{(t)}\geq 0$ and the negative case follows a similar proof. By Assumption \ref{assumption:3}, we know $|\sigma_k^{(t)}(\bm{x}^{(t)\text{new}}_k)-\sigma_k^{(t)}(\bm{x}^{(t)}_k)|\leq \mathcal{O}(\frac{1}{(\log (2+t))^{0.5+\epsilon}})$. Therefore $\sigma_k^{(t)}(\bm{x}^{(t)\text{new}}_k)$ is upper bounded by $\mathcal{O}(\frac{1}{(\log (2+t))^{0.5+\epsilon}})+\sigma_k^{(t)}(\bm{x}^{(t)}_k)$. When $\mu^{(t)}_k(\bm{x}^{(t)}_k)\geq \mu^{(t)}_k(\bm{x}^{(t)\text{new}}_k)$, we take the upper bound of $\sigma_k^{(t)}(\bm{x}^{(t)\text{new}}_k)$ and derive that 
\begin{align*}
    &\left|\sigma_k^{(t)}(\bm{x}^{(t)\text{new}}_k)\mu^{(t)}_k(\bm{x}^{(t)}_k) - \sigma_k^{(t)}(\bm{x}^{(t)}_k)\mu^{(t)}_k(\bm{x}^{(t)\text{new}}_k)\right|\\
    &=\left|(\sigma_k^{(t)}(\bm{x}^{(t)}_k)+\mathcal{O}(\frac{1}{(\log (2+t))^{0.5+\epsilon}}))\mu^{(t)}_k(\bm{x}^{(t)}_k) - \sigma_k^{(t)}(\bm{x}^{(t)}_k)\mu^{(t)}_k(\bm{x}^{(t)\text{new}}_k)\right|\\
    &=\left|\sigma_k^{(t)}(\bm{x}^{(t)}_k)(\mu^{(t)}_k(\bm{x}^{(t)}_k)-\mu^{(t)}_k(\bm{x}^{(t)\text{new}}_k)) +\mathcal{O}(\frac{1}{(\log (2+t))^{0.5+\epsilon}})\mu^{(t)}_k(\bm{x}^{(t)}_k)  \right|.
    % &\leq \mathcal{O}(\frac{1}{(\log (2+t))^{0.5+\epsilon}}).
\end{align*}
By Assumption \ref{assumption:3}, we know $\sigma_k^{(t)}(\bm{x}^{(t)}_k)$ is finite. Otherwise, the bound on $|\sigma_k^{(t)}(\bm{x}^{(t)}_k)-\sigma_k^{(t)}(\bm{x}^{(t)\text{new}}_k)|$ is infinite. Also, $\sigma_k^{(t)}(\bm{x}^{(t)}_k)\neq 0$. Therefore, we know $\sigma_k^{(t)}(\bm{x}^{(t)}_k)(\mu^{(t)}_k(\bm{x}^{(t)}_k)-\mu^{(t)}_k(\bm{x}^{(t)\text{new}}_k))$ is of order $ \mathcal{O}(\frac{1}{(\log (2+t))^{0.5+\epsilon}})$. By Corollary \ref{lemma:1}, we know $\mu^{(t)}_k(\bm{x}^{(t)}_k)$ is finite. As a result, $\left|\sigma_k^{(t)}(\bm{x}^{(t)\text{new}}_k)\mu^{(t)}_k(\bm{x}^{(t)}_k) - \sigma_k^{(t)}(\bm{x}^{(t)}_k)\mu^{(t)}_k(\bm{x}^{(t)\text{new}}_k)\right|$ is of order  $ \mathcal{O}(\frac{1}{(\log (2+t))^{0.5+\epsilon}})$.

When $\mu^{(t)}_k(\bm{x}^{(t)}_k)\leq \mu^{(t)}_k(\bm{x}^{(t)\text{new}}_k)$, we take the lower bound of $\sigma_k^{(t)}(\bm{x}^{(t)\text{new}}_k)$ and the same conclusion holds. 

As a result, we show that $|z_k^{(t)}(\bm{x}^{(t)}_k)-z_k^{(t)}(\bm{x}^{(t)\text{new}}_k)|$ is of order $ \mathcal{O}(\frac{1}{(\log (2+t))^{0.5+\epsilon}})$. Therefore, Eq. \eqref{inequality:1} gives $\texttt{EI}^{(t)}_k(\bm{x}_k^{(t)})\leq \texttt{EI}^{(t)}_k(\bm{x}^{(t)\text{new}}_k)+\mathcal{O}(\frac{1}{(\log (2+t))^{0.5+\epsilon}})$. Additionally, we know $\texttt{EI}^{(t)}_k(\bm{x}_k^*)\leq\texttt{EI}^{(t)}_k(\bm{x}_k^{(t)})$ since $\bm{x}_k^{(t)}$ is the maximizer of $\texttt{EI}_k$. This gives $\texttt{EI}^{(t)}_k(\bm{x}_k^*)\leq \texttt{EI}^{(t)}_k(\bm{x}^{(t)\text{new}}_k)+\mathcal{O}(\frac{1}{(\log (2+t))^{0.5+\epsilon}})$. Therefore, with probability $1-\delta$,
\begin{align*}
    \text{A}&\leq \texttt{EI}^{(t)}_k(\bm{x}^{(t)\text{new}}_k)+\mathcal{O}(\frac{1}{(\log (2+t))^{0.5+\epsilon}})+\sqrt{\beta^{(t)}_{k}}\sigma^{(t)}_{k}(\bm{x}_k^*)\\
    &= \sigma_k^{(t)}(\bm{x}^{(t)\text{new}}_k)\tau(z_k^{(t)}(\bm{x}^{(t)\text{new}}_k)) +\sqrt{\beta^{(t)}_{k}}\sigma^{(t)}_{k}(\bm{x}_k^*)+\mathcal{O}(\frac{1}{(\log (2+t))^{0.5+\epsilon}}),
\end{align*}
where the last equality holds by Lemma \ref{lemma:2}.

\paragraph{Bounding Term B} We first decompose B as
\begin{align*}
    \text{B}&=y^{*(t)}_{k}- f_{k}(\bm{x}^{(t)\text{new}}_k)=y^{*(t)}_{k}-\mu_k^{(t)}(\bm{x}^{(t)\text{new}}_k) +\mu_k^{(t)}(\bm{x}^{(t)\text{new}}_k)- f_{k}(\bm{x}^{(t)\text{new}}_k).
\end{align*}
Using Lemma \ref{lemma:1}, we know $\mu_k^{(t)}(\bm{x}^{(t)\text{new}}_k)- f_{k}(\bm{x}^{(t)\text{new}}_k)\leq \sigma^{(t)}_{k}(\bm{x}^{(t)\text{new}}_k)\sqrt{\beta_k^{(t)}}$.
Therefore, we can bound term B as
\begin{align*}
    \text{B}&=y^{*(t)}_{k}- f_{k}(\bm{x}^{(t)\text{new}}_k)=y^{*(t)}_{k}-\mu_k^{(t)}(\bm{x}^{(t)\text{new}}_k) +\mu_k^{(t)}(\bm{x}^{(t)\text{new}}_k)- f_{k}(\bm{x}^{(t)\text{new}}_k)\\
    &\leq \sigma^{(t)}_{k}(\bm{x}^{(t)\text{new}}_k)(-z_k^{(t)}(\bm{x}^{(t)\text{new}}_k)) +  \sigma^{(t)}_{k}(\bm{x}^{(t)\text{new}}_k)\sqrt{\beta_k^{(t)}}\\
    &\leq\sigma^{(t)}_{k}(\bm{x}^{(t)\text{new}}_k)\left(\tau(-z_k^{(t)}(\bm{x}^{(t)\text{new}}_k))-\tau(z_k^{(t)}(\bm{x}^{(t)\text{new}}_k))+\sqrt{\beta_k^{(t)}} \right),
\end{align*}
where the last inequality uses the fact that $z=\tau(z)-\tau(-z)$. 

\paragraph{Compiling A and B}

Given all the previous information, we reach
\begin{align*}
    r_{k}^{(t)}&=\text{A}+\text{B}\\
    &\leq\sigma^{(t)}_{k}(\bm{x}^{(t)\text{new}}_k)\left(\tau(-z_k^{(t)}(\bm{x}^{(t)\text{new}}_k))+\sqrt{\beta_k^{(t)}} \right) +\sqrt{\beta^{(t)}_{k}}\sigma^{(t)}_{k}(\bm{x}_k^*)+\mathcal{O}(\frac{1}{(\log (2+t))^{0.5+\epsilon}}).
\end{align*}
Now, by Lemma \ref{lemma:4}, we have, with probability $1-\delta$,
\begin{align*}
    r_{k}^{(t)}\leq \sigma^{(t)}_{k}(\bm{x}^{(t)\text{new}}_k)\left[\sqrt{\beta_k^{(t)}}+1+\sqrt{C}\right]+\sqrt{\beta^{(t)}_{k}}\sigma^{(t)}_{k}(\bm{x}_k^*)+\mathcal{O}(\frac{1}{(\log (2+t))^{0.5+\epsilon}}),
\end{align*}
where $C$ is defined in Lemma \ref{lemma:4}. We next aim to further bound $\sum_{t=1}^Tr_{k}^{(t)}$ using Lemma \ref{lemma:5}.

Let $\delta=\frac{\delta_1}{T}$, where $\delta_1\in(0,T)$. By Lemma \ref{lemma:5}, and using the fact that $\left[\sqrt{\beta_k^{(t)}}+1+\sqrt{C}\right]^2\geq 3(\beta_k^{(t)}+1+C)$, we have
\begin{align*}
    &\sum_{t=1}^T\left(\sigma^{(t)}_{k}(\bm{x}^{(t)\text{new}}_k)\left[\sqrt{\beta_k^{(t)}}+1+\sqrt{C}\right]\right)^2\\
    &\leq3\sum_{t=1}^T(\sigma^{(t)}_{k}(\bm{x}^{(t)\text{new}}_k))^2\left[\beta_k^{(t)}+1+C \right]\\
    &\leq3\sum_{t=1}^T(\sigma^{(t)}_{k}(\bm{x}^{(t)\text{new}}_k))^2\left[\beta_k^{(T)}+1+C \right]& & \text{$\beta_k^{(t)}$ is non-decreasing}\\
    &\leq\frac{6\left[\beta_k^{(T)}+1+C \right]\gamma_{k,T}}{\log(1+v_k^{-2})} & & \text{by Lemma \ref{lemma:5}},
\end{align*}
with probability $(1-\frac{\delta_1}{T})^T$, where $\gamma_{k,T}$ is defined in Lemma \ref{lemma:5}. Therefore, by the Cauchy-Schwartz inequality, we have
\begin{align*}
&\sum_{t=1}^T\left(\sigma^{(t)}_{k}(\bm{x}^{(t)new}_k)\left[\sqrt{\beta_k^{(t)}}+1+\sqrt{C}\right]\right)\\
&\leq\sqrt{T}\sqrt{\sum_{t=1}^T\left(\sigma^{(t)}_{k}(\bm{x}^{(t)\text{new}}_k)\left[\sqrt{\beta_k^{(t)}}+1+\sqrt{C}\right]\right)^2}\\
&\leq \sqrt{T}\sqrt{\frac{6\left[\beta_k^{(T)}+1+C \right]\gamma_{k,T}}{\log(1+v_k^{-2})}}.  & & \text{by Lemma \ref{lemma:5}}
\end{align*}
Similarly, we can show that, with probability $(1-\frac{\delta_1}{T})^T$,
\begin{align*}
\sum_{t=1}^T(\sqrt{\beta^{(t)}_{k}}\sigma^{(t)}_{k}(\bm{x}_k^{(t)\text{new}}))\leq\sqrt{T}\sqrt{\beta^{(T)}_{k}\sum_{t=1}^T\sigma^{2(t)}_{k}(\bm{x}^{(t)\text{new}}_k)}\leq\sqrt{\frac{2T\beta_k^{(T)}\gamma_{k,T}}{\log(1+v_k^{-2})}}.
\end{align*}
Finally, we obtain, with probability $(1-\frac{\delta_1}{T})^T$,
\begin{align*}
    \sum_{t=1}^Tr_{k}^{(t)}\leq \sqrt{T}\sqrt{\frac{6\left[\beta_k^{(T)}+1+C \right]\gamma_{k,T}}{\log(1+v_k^{-2})}}+\sqrt{\frac{2T\beta_k^{(T)}\gamma_{k,T}}{\log(1+v_k^{-2})}}+\sum_{t=1}^T\mathcal{O}(\frac{1}{(\log (2+t))^{0.5+\epsilon}}).
\end{align*}
From Lemma \ref{lemma:1}, we know $\beta_k^{(T)}\sim\mathcal{O}((\log (T/\delta))^3)$, where $\delta=\frac{\delta_1}{T}$, $\forall\delta_1\in(0,T)$. Furthermore, \cite{contal2014gaussian} have shown that $\gamma_{k,T}\sim\mathcal{O}((\log T)^{D+1})$ if a squared exponential kernel is used. Therefore, we can show that $\sum_{t=1}^Tr_{k}^{(t)}\sim\mathcal{O}(\sqrt{T\times(\log T)^{D+4}})$. As a result, $\lim_{T\to\infty}\frac{R_{k,T}}{T}=0$. 

\end{proof}

\section{Supporting Lemmas and Some Useful Equations}
\label{sec:2}

%%%%%%%%%%%%%%%%%%%%%

Before proving Lemma \ref{lemma:2}, we first need a Lemma from \cite{srinivas2009gaussian}.

\begin{lemma}
\label{lemma:1}
\citep{srinivas2009gaussian} Let $\delta\in(0,1)$. Then, $\forall t$ and $\forall \bm{x}\in\mathbb{R}^D$, there exists a non-decreasing positive sequence $\{\beta_k^{(t)}\}_t$ such that $\beta_k^{(t)}\sim\mathcal{O}((\log \frac{t}{\delta})^3)$ and 
\begin{align*}
    \mathbb{P}\left( |\mu^{(t)}_k(\bm{x})-f_{k}(\bm{x})|\leq \sqrt{\beta_k^{(t)}}\sigma_k^{(t)}(\bm{x}) \right)\geq 1-\delta.
\end{align*}
\end{lemma}

\begin{lemma}
\label{lemma:2}
Let $I_k^{(t)}(\bm{x})=(f_{k}(\bm{x}) -y^{*(t)}_{k})^+$. We have ${\normalfont\texttt{EI}}^{(t)}_k(\bm{x})=\sigma_k^{(t)}(\bm{x})\tau(z_k^{(t)}(\bm{x}))$, where $\tau(z_k^{(t)}(\bm{x}))\coloneqq z_k^{(t)}(\bm{x})\Phi(z_k^{(t)}(\bm{x}))+\phi(z_k^{(t)}(\bm{x}))$ and $z_k^{(t)}(\bm{x})=\frac{\mu^{(t)}_k(\bm{x})-y^{*(t)}_{k}}{\sigma_k^{(t)}(\bm{x})}$. Additionally, with probability $1-\delta$ ($\delta\in(0,1)$),  ${\normalfont\texttt{EI}}^{(t)}_k(\bm{x})\geq I_k^{(t)}(\bm{x})-\sqrt{\beta_k^{(t)}}\sigma_k^{(t)}(\bm{x})$.
\end{lemma}

\begin{proof}
% Let $z_k^{(t)}(\bm{x})=\frac{\mu^{(t)}_k(\bm{x})-y^{*(t)}_{k}}{\sigma_k^{(t)}(\bm{x})}$, 

We have, with probability at least $1-\delta$,
\begin{align*}
    \texttt{EI}^{(t)}_k(\bm{x})&=\sigma_k^{(t)}(\bm{x})\phi(z_k^{(t)}(\bm{x}))+(\mu^{(t)}_k(\bm{x})-y^{*(t)}_{k})\Phi(z_k^{(t)}(\bm{x}))\\
    &=\sigma_k^{(t)}(\bm{x})\left[z_k^{(t)}(\bm{x})\Phi(z_k^{(t)}(\bm{x}))+\phi(z_k^{(t)}(\bm{x}))\right]\\
    &=\sigma_k^{(t)}(\bm{x})\tau(z_k^{(t)}(\bm{x})).
\end{align*}
Therefore, we have shown that ${\normalfont\texttt{EI}}^{(t)}_k(\bm{x})=\sigma_k^{(t)}(\bm{x})\tau(z_k^{(t)}(\bm{x}))$.
% where $\tau(z_k^{(t)}(\bm{x}))= z_k^{(t)}(\bm{x})\Phi(z_k^{(t)}(\bm{x}))+\phi(z_k^{(t)}(\bm{x}))$, 

Now, let $q_k^{(t)}(\bm{x})=\frac{f_{k}(\bm{x})-y^{*(t)}_{k}}{\sigma_k^{(t)}(\bm{x})}$, we know
\begin{align*}
    z_k^{(t)}(\bm{x})-q_k^{(t)}(\bm{x})&= \frac{\mu^{(t)}_{k}(\bm{x})-y^{*(t)}_{k}}{\sigma_k^{(t)}(\bm{x})}- \frac{f_{k}(\bm{x})-y^{*(t)}_{k}}{\sigma_k^{(t)}(\bm{x})}\\
    &=\frac{\mu^{(t)}_{k}(\bm{x})-f_{k}(\bm{x})}{\sigma_k^{(t)}(\bm{x})}\\
    &\geq -\frac{\sqrt{\beta_k^{(t)}}\sigma_k^{(t)}(\bm{x}) }{\sigma_k^{(t)}(\bm{x})}. & & \text{By Lemma \ref{lemma:1}}
\end{align*}
Therefore, $z_k^{(t)}(\bm{x})\geq q_k^{(t)}(\bm{x})-\sqrt{\beta_k^{(t)}}$. We know $\tau(\cdot)$ is a non-decreasing function. Therefore, we have
\begin{align*}
    \texttt{EI}^{(t)}_k(\bm{x})&=\sigma_k^{(t)}(\bm{x})\tau(z_k^{(t)}(\bm{x}))\\
    &\geq\sigma_k^{(t)}(\bm{x})\tau\left(q_k^{(t)}(\bm{x})-\sqrt{\beta_k^{(t)}}\right)\\
    &\geq\sigma_k^{(t)}(\bm{x})\left(q_k^{(t)}(\bm{x})-\sqrt{\beta_k^{(t)}}\right).
\end{align*}
The last inequality holds due to $\tau(z)\geq z, \forall z$. If $I_k^{(t)}(\bm{x})=0,$ then we know $q_k^{(t)}(\bm{x})\leq 0$ and $\sigma_k^{(t)}(\bm{x})\left(q_k^{(t)}(\bm{x})-\sqrt{\beta_k^{(t)}}\right)<0$. Therefore, we have $\texttt{EI}^{(t)}_k(\bm{x})\geq I_k^{(t)}(\bm{x})-\sqrt{\beta_k^{(t)}}\sigma_k^{(t)}(\bm{x})$. If $I_k^{(t)}(\bm{x})>0$, then we also have
\begin{align*}
    \texttt{EI}^{(t)}_k(\bm{x})&\geq\sigma_k^{(t)}(\bm{x})\left(q_k^{(t)}(\bm{x})-\sqrt{\beta_k^{(t)}}\right)\\
    &= I_k^{(t)}(\bm{x})-\sqrt{\beta_k^{(t)}}\sigma_k^{(t)}(\bm{x}).
\end{align*}
\end{proof}

\begin{lemma}
\label{lemma:4}
Given a constant $\kappa>0$, if $y^{*(t)}_{k}-\mu_k^{(t)}(\bm{x}^{(t)\textup{new}}_k)\geq 0$, then we have $y^{*(t)}_{k}-\mu_k^{(t)}(\bm{x}^{(t)\textup{new}}_k)\leq \sigma^{(t)}_{k}(\bm{x}^{(t)\textup{new}}_k)\sqrt{C}$, where $C=\log[\frac{1}{2\pi\kappa^2}]$. Furthermore, We have $\tau(-z_k^{(t)}(\bm{x}^{(t)new}_k))\leq 1+\sqrt{C}$.
\end{lemma}

\begin{proof}
Suppose the algorithm is not stopped. By Lemma \ref{lemma:2}, we know
\begin{align*}
    \kappa\leq\texttt{EI}^{(t)}_k(\bm{x}^{(t)\text{new}}_k)=\sigma_k^{(t)}(\bm{x}^{(t)\text{new}}_k)\tau(z_k^{(t)}(\bm{x}^{(t)\text{new}}_k)).
\end{align*}
When $y^{*(t)}_{k}-\mu_k^{(t)}(\bm{x}^{(t)\textup{new}}_k)\geq 0$ (i.e., $z_k^{(t)}(\bm{x}^{(t)\text{new}}_k)\leq 0$), we have 
\begin{align*}
    \tau(z_k^{(t)}(\bm{x}_k^{(t)\text{new}}))\leq \phi(z_k^{(t)}(\bm{x}_k^{(t)\text{new}})).
\end{align*}
This implies 
\begin{align*}
    \kappa&\leq\sigma_k^{(t)}(\bm{x}^{(t)\text{new}}_k)\frac{1}{\sqrt{2\pi}}\exp\left[-\frac{1}{2}(z_k^{(t)}(\bm{x}_k^{(t)\text{new}}))^2\right]\\
    &\Rightarrow (z_k^{(t)}(\bm{x}_k^{(t)\text{new}}))^2\leq 2\log\frac{\sigma_k^{(t)}(\bm{x}^{(t)\text{new}}_k)}{\sqrt{2\pi}\kappa}\leq \log[\frac{1}{2\pi\kappa^2}]=C. & & \text{Since $\mathcal{K}_k(\bm{x}_1,\bm{x}_2)\leq 1$}
\end{align*}
Therefore, $(z_k^{(t)}(\bm{x}^{(t)\text{new}}))^2\leq C$. By the definition of $z_k^{(t)}(\bm{x}^{(t)\text{new}})$, we have  $y^{*(t)}_{k}-\mu_k^{(t)}(\bm{x}^{(t)\textup{new}}_k)\leq \sigma^{(t)}_{k}(\bm{x}^{(t)\textup{new}}_k)\sqrt{C}$.

Now, to prove $\tau(-z_k^{(t)}(\bm{x}^{(t)\text{new}}_k))\leq 1+\sqrt{C}$, we consider two situations. When $y^{*(t)}_{k}-\mu_k^{(t)}(\bm{x}^{(t)\text{new}}_k)\geq 0$, we have
\begin{align*}
    \tau(-z_k^{(t)}(\bm{x}^{(t)\text{new}}_k))\leq 1+\sqrt{C}
\end{align*}
using the fact that $\tau(z)\leq 1+z, \forall z\geq 0$. When $y^{*(t)}_{k}-\mu_k^{(t)}(\bm{x}^{(t)\text{new}}_k)\leq 0$, we know $\forall z\leq0, \tau(z)\leq\phi(z)\leq 1$. Therefore,
\begin{align*}
    \tau(-z_k^{(t)}(\bm{x}^{(t)\text{new}}_k))\leq 1.
\end{align*}
As a result, we conclude that $\tau(-z_k^{(t)}(\bm{x}^{(t)\text{new}}_k))\leq 1+\sqrt{C}$.
\end{proof}

%%%%%%%%%%%%%%%%%%%%%

\begin{lemma}
\label{lemma:5}
Let $\text{J}^{(t)}_k$ be the mutual information at iteration $t$ and denote by $\gamma_{k,T}=\max \text{J}^{(t)}_k$ the maximum information gain after $T$ iterations. We have $\sum_{t=1}^T\sigma^{2(t)}_{k}(\bm{x}^{(t)\text{new}}_k)\leq\frac{2}{\log(1+\sigma_k^{-2})}\gamma_{k,T}$.
\end{lemma}
\begin{proof}
By definition, $\text{J}^{(t)}_k=\frac{1}{2}\log(1+v_k^{-2}\sigma^{2(t)}_{k}(\bm{x}^{(t)\text{new}}_k))$ \citep{contal2014gaussian}. Given a constant $s^{2(t)}=v_k^{-2}\sigma^{2(t)}_{k}(\bm{x}^{(t)\text{new}}_k)\leq v_k^{-2}$, we have $1\leq\frac{s^{(t)}}{\log(1+s^{(t)})}\leq\frac{1}{v_k^{2}\log(1+v_k^{-2})}$.

Therefore, 
\begin{align*}
    \sum_{t=1}^T\sigma^{2(t)}_{k}(\bm{x}^{(t)\text{new}}_k)&=\sum_{t=1}^Tv_k^{2}s^{(t)}\\
    &\leq\sum_{t=1}^Tv_k^2\left[\frac{\log(1+s^{(t)})}{v_k^2\log(1+v_k^{-2})}\right]\\
    &=\frac{2}{\log(1+v_k^{-2})}\frac{1}{2}\sum_{t=1}^T\log(1+v_k^{-2}\sigma^{2(t)}_{k}(\bm{x}^{(t)\text{new}}_k))\\
    &= \frac{2}{\log(1+v_k^{-2})}\text{J}^{(t)}_k\\
    &\leq \frac{2}{\log(1+v_k^{-2})} \gamma_{k,T}.
\end{align*}
\end{proof}

\newpage

\section{Details on Testing Functions}

In this section, we provide more details on the testing function.

\paragraph{Shekel-$10$} 
Shekel-$10$ is a four-dimensional ($D=4$) with $10$ local minima. It has a functional form
\begin{align*}
    f(\bm{x})=-\sum_{i=1}^{10}\left(\sum_{d=1}^4(\tensor[^d]{x}{}-F_{di})^2+\xi_i\right)^{-1},
\end{align*}
where $F_{di}$ is the $(d,i)$-th component of a matrix $\bm{F}$ and $\xi_i$ is the $i$-th component of a vector $\bm{\xi}$, and
\begin{align*}
    \bm{\xi}=\frac{1}{10}(1,2,2,4,4,6,3,7,5,5)^\intercal,
    \bm{F}=\begin{bmatrix}
    4 & 1 & 8 & 6 & 3 & 2 & 5 & 8 & 6 & 7\\
    4 & 1 & 8 & 6 & 7 & 9 & 3 & 1 & 2 & 3.6\\
    4 & 1 & 8 & 6 & 3 & 2 & 5 & 8 & 6 & 7\\
    4 & 1 & 8 & 6 & 7 & 9 & 3 & 1 & 2 & 3.6
    \end{bmatrix}, 
\end{align*}
$\bm{x}=(\tensor[^1]{x}{},\ldots,\tensor[^d]{x}{}\ldots,\tensor[^{10}]{x}{})^\intercal\in[0,10]^D$.

\paragraph{Harmann-$6$} Hartmann-$6$ is a six-dimensional ($D=6$) function in a form of 
\begin{align*}
    f(\bm{x})=\sum_{i=1}^4\alpha_i\exp\left(-\sum_{d=1}^6A_{id}(\tensor[^d]{x}{}-P_{id})^2\right),
\end{align*}
where 
\begin{align*}
    A&=\begin{bmatrix}10&3&17&3.5&1.7&8\\0.05&10&17&0.1&8&14\\3&3.5&1.7&10&17&8\\17&8&0.05&10&0.1&14\end{bmatrix}, P = \begin{bmatrix}0.1312&0.1696&0.5569&0.0124&0.8283&0.5886\\0.2329 &0.4135 &0.8307&0.3736&0.1004&0.9991\\0.2348 &0.1451& 0.3522&0.2883&0.3047&0.6650\\0.4047& 0.8828 &0.8732&0.5743&0.1091&0.0381\end{bmatrix},\\
    &\alpha_1=1.0, \alpha_2=1.2, \alpha_3=3.0, \alpha_4=3.2, \bm{x}\in[0,1]^6. 
\end{align*}
For each client, we generate $a_1\sim\texttt{Uniform}(0.5, 2)$ and $a_2,a_3\sim\mathcal{N}(0,1)$. 

% \section{More Results}

% In this section, we present the experimental results from Sec. 5.4 in the main paper. 

% We test \texttt{CSDC} and other benchmarks using more testing functions. Similar to previous sections, we set $f_k(\bm{x})=a_1f(\bm{x}+\text{vec}(a_3))+a_2$ and create $K=10$ clients.

\paragraph{Branin} Branin is a two-dimensional ($D=2$) function that has the form 
\begin{align*}
    f(\bm{x})=(\tensor[^2]{x}{}-\frac{5.1}{4\pi^2}\tensor[^1]{x}{}^2+\frac{5}{\pi}\tensor[^1]{x}{}-6)^2+10(1-\frac{1}{8\pi})\cos(\tensor[^1]{x}{})+10,
\end{align*}
where $\bm{x}=(\tensor[^1]{x}{},\tensor[^2]{x}{})^\intercal$ and $\tensor[^1]{x}{}\in[-5,10], \tensor[^2]{x}{}\in[0,15]$. For each client, we set $a_1\sim\texttt{Uniform}(0.5, 1)$ and $a_2,a_3\sim\mathcal{N}(0,1)$.

% \paragraph{Levy} Levy is an $D$-dimensional function in a form of
% \begin{align*}
%     f(\bm{x})=\sin^2(\pi w_1)+\sum_{d=1}^{D-1}(w_d-1)^2[1+10\sin^2(\pi w_d+1)]+(w_D-1)^2[1+\sin^2(2\pi w_D)],
% \end{align*}
% where $w_d=1+\frac{x_d-1}{4}, \forall d\in[D]$ and $\bm{x}\in[-10,10]^D$.  For each client, we set $a_1\sim\texttt{Uniform}(1, 2)$ and $a_2\sim\mathcal{N}(0.5,1)$.

\paragraph{Ackley-$D$} Ackley-$D$ is a $D$-dimensional function. When $D=2$, this function is characterized by a nearly flat outer region and a large valley at the center. This property renders the Ackley function very challenging to be optimized. Given an arbitrary dimension $D$, Ackley-$D$ has the following function form:
\begin{align*}
    f(\bm{x})=-20\exp\left(-0.2\sqrt{\frac{1}{D}\sum_{d=1}^D\tensor[^d]{x}{}^2} \right)-\exp\left(\frac{1}{D}\sum_{i=1}^D\cos(2\pi \tensor[^d]{x}{})\right)+20+\exp(1),     
\end{align*}
where $\tensor[^d]{x}{}\in[-32.768,32.768],$ for $d=1,\ldots,D$. For each client, we set $a_1\sim\texttt{Uniform}(1, 2)$ and $a_2,a_3\sim\mathcal{N}(0.5,1)$.

\bibliography{mybib}
\bibliographystyle{apalike}
\end{document}